\documentclass{article}

\usepackage[biblatex]{smacros}
\usepackage[utf8]{inputenc}
\usepackage[T1]{fontenc}
\usepackage{hyperref}
\usepackage{url}
\usepackage{booktabs}
\usepackage{amsmath,amsthm,amssymb,amsfonts}
\usepackage{nicefrac}
\usepackage{microtype}
\usepackage{xcolor}
\usepackage{enumitem}
\usepackage{mathtools}
\usepackage{algorithm}
\usepackage{algpseudocode}
\usepackage{colortbl}
\usepackage{multirow}
\usepackage{graphicx}
\usepackage{subcaption}

\newtheorem{assumption}{Assumption}

\title{Matching Markets meet Cumulative Prospect Theory:\\
Towards Optimal and Adversarially Robust Learning}

\author{Ananya Kunisetty \qquad Avishek Ghosh\\[4pt]
Indian Institute of Technology Bombay, Mumbai 400076, India\\
\texttt{ananya.kunisetty@gmail.com, avishek\_ghosh@iitb.ac.in}}
\date{}

\begin{document}

\maketitle

\begin{abstract}
We study a multi-agent multi-armed bandit problem in the competitive setup with two-sided matching markets  under a human centric decision making model. To capture human preferences, we use cumulative prospect theory (CPT) that weighs the actions of the agent in a nonlinear fashion using a ($\alpha$-H\"older continuous) weight function. CPT has been widely used in behavioral economics and risk sensitive machine learning to emulate human preferences. We analyze the state-of-the-art learning algorithm \cite{kong2023player} with CPT weight distorted rewards and derive a regret of $\mathcal{O}(K\log T \left(\frac{1}{\Delta}\right)^{2/\alpha})$,
where $K$ denotes the number of arms, $T$ is the learning horizon, and $\Delta$ represents (suitably defined) players’ minimum preference gap. Obtaining a regret lower bound, we notice that the dependence on $\Delta$ is sub-optimal. We further improve this regret by judiciously selecting the active set of arms during exploration, which achieves the optimal regret guarantees (matching the lower bound) in the setting where the number of arms $K$ is significantly larger than the number of players $N$. In addition, we also consider adversarial markets where the observed rewards of the agents may be corrupted. We propose and analyze algorithms for robust markets with CPT as risk sensitive measure in both settings where the total corruption budget is known and where it is unknown, and establish logarithmic player-optimal regret guarantees in both cases.
\end{abstract}


\section{Introduction}
\label{sec:intro}

Online matching markets—such as Amazon Mechanical Turk, Upwork, ride-sharing platforms (like Uber) and labor markets connect a demand side (agents or players) with a supply side (resources). For example, businesses hiring workers on Mechanical Turk or Upwork, customers requesting rides on Uber form the demand side, while freelancers, crowdworkers, drivers represent the supply side. Agents repeatedly choose among resources based on their preferences. Since resources are typically capacity constrained, agents compete for access while simultaneously navigating uncertainty on the quality of each resource  through interactions. These interactions are often modeled as a bipartite graph, where both agents and resources have preferences over the opposite side, though these preferences are initially unknown. Agents must therefore learn these preferences through minimal interactions and subsequently obtain an optimal stable matching. Once preferences are learned, algorithms such as the Deferred Acceptance algorithm proposed by David Gale and Lloyd Shapley can be used to compute the optimal stable matching.

Learning in matching markets has received considerable interest, especially from the lens of a multi-agent MAB (multi-armed bandit) framework \cite{liu2021bandit,liu2020competing,kong2023player,pmlr-v130-sankararaman21a,basu2021beyond,pagare2024explore,ghosh_nonstationary,meena_market}. In the above works, a variation of matching markets is studied where each player has unknown preference over arms whereas arms can be certain of their preferences over players through some known utility like payments (in case of labor markets) or some global ranking (in finance). Denoting $K$ as the number of arms, $N$ as the number of players, $T$ as the time horizon with $K \geq N$. This is done to ensure each player has a chance of being matched (see \cite{liu2020competing,pmlr-v130-sankararaman21a,basu2021beyond,kong2023player}). Moreover, let $\Delta$ be the minimum preference gap between the first $N+1$ arms across all players. With this, the objective is to find the most suitable (player optimal) stable matching for all the players. In previous works (\cite{liu2021bandit,kong2023player,pmlr-v130-sankararaman21a,basu2021beyond,ghosh_nonstationary,pagare2024explore,meena_market}) this is ensured by defining an appropriate regret defined as the mean reward of the \emph{best} arm in hindsight minus the mean reward obtained from a learning algorithm in presence of competition from other agents, and showing that the regret scales sublinearly with the learning horizon.

However, in real-world applications—ranging from labor markets to financial portfolio optimization—the mean (average) does not satisfactorily capture the merits of certain actions, especially when humans are involved in the decision making (see \cite{tan2022surveyriskawaremultiarmedbandits} for a recent survey). Also, \cite{Rockafellar2000OptimizationOC} demonstrates that conventional mean-based optimization of a Nikkei option portfolio was unable to identify hidden market crashes, which become apparent only when tail-risk measures are considered.

In many such applications, risk is a crucial consideration, and decision-makers often prefer a risk-sensitive performance criterion rather than relying solely on the average reward. As a result, extensive research has examined a variety of risk measures, including mean–variance \cite{e5a1bb8f-41b7-35c6-95cd-8b366d3e99bc}, Value-at-Risk (VaR), Conditional Value-at-Risk (CVaR) \cite{https://doi.org/10.1111/1467-9965.00068,Rockafellar2000OptimizationOC}, spectral risk measures \cite{ACERBI20021505}, and utility-based shortfall risk \cite{Hu2018UtilitybasedSR}. Although these measures satisfy desirable properties such as coherence or convexity, they fail to capture the nuances of human decisions like loss aversion and asymmetric preferences. Such behavioral features are more effectively modeled by the celebrated (and Nobel Prize–winning) framework of Cumulative Prospect Theory (CPT) \cite{tversky1992advances} in behavioral economics.

Consider the ``Asian Disease Problem'' \cite{Kahneman2011-KAHTFA-2}, in which a decision maker aims to maximize a Public Health Score $r$, representing the number of lives saved. The agent must choose between two policies: a \emph{Safe} protocol that guarantees $r=200$, and a \emph{Risky} protocol that yields $r=600$ with probability $35\%$ and $r=0$ with probability $65\%$. Under a risk-neutral criterion based on expected reward, the risky protocol is optimal because its expected value is higher ($\mu_{\text{Risky}} = 210$) than that of the safe protocol ($\mu_{\text{Safe}} = 200$). However, empirical studies show that most human decision-makers reject this gamble and instead choose the safe protocol \cite{tversky1992advances}.

Cumulative Prospect Theory (CPT) resolves this discrepancy by evaluating outcomes through a distorted value rather than a simple expectation, using a suitably chosen non-linear weighting function. In this example, the CPT weighting function disproportionately emphasizes the probability of the adverse event $\{r=0\}$, thereby penalizing the possibility of saving no patients. As a result, CPT correctly identifies the safe protocol as the \emph{subjectively optimal} choice.

The study of CPT based risk measure for multi-armed bandits (MABs) has been limited. \cite{DBLP:journals/corr/GopalanAFM16} introduces the notion of CPT based regret for simple $K$ armed bandits and linear bandits, whereas \cite{pmlr-v258-tatli25a} uses a generic framework to handle monotone as well as non-monotone weight distorted rewards including CPT. The above works consider a single player multi-armed bandit setup and this leads to an interesting open problem:

{\centering \itshape Is it possible to learn in a multi-player (competitive or collaborative) MAB with CPT risk measure? \par}

In this paper, we answer the above question in affirmative. We consider multi-player MAB in the competitive setup of matching markets. Using CPT as risk measure, we define an appropriate distorted regret through a non-linear H\"older continuous function, similar to \cite{DBLP:journals/corr/GopalanAFM16} albeit in the competitive market setting. The objective here is to converge to a (player optimal) stable matching based on this weighted (distorted) reward which may be quite different from the standard matching based on reward average (see Figure~\ref{fig:preference-reversal} for an example). We propose and analyze CPT based state of the art learning algorithm (\cite{kong2023player}) namely CPT-Explore-Then-Gale-Shapley (\texttt{CPT-ETGS}) and obtain sublinear (logarithmic) regret bound.

By proving a lower bound on CPT based regret in the matching markets setup, we observe that the regret of \texttt{CPT-ETGS} is sub-optimal. Hence, we modify the learning algorithm by carefully choosing the set of arms during the exploration phase of  \texttt{CPT-ETGS}. We dub the resulting algorithm Improved \texttt{CPT-ETGS}. As a result, we show that in large markets, where the number of arms $K$ is significantly larger than $N$, Improved \texttt{CPT-ETGS} achieves the lower bound on regret and hence optimal.

Moreover, in a practical competitive market, there might be manipulation (or corruption) of rewards received by the players. The effect of such corruption aggravates in the presence of weight distortion. Notably, online learning in the presence of corruption is well explored in applications like feedback contamination (click fraud) in online advertisement \cite{zhang_click,oentaryo_click}, non-stationary online learning \cite{fitness_sankararaman22a,robust_abishek}. Furthermore, adversarial bandits is a well studied problem in the literature albeit in a single player setup \cite{niss2020you,arora2012online,yang_adversarial}. Recently, \cite{wubandit} study the matching markets problem in the presence of (bounded) corruption and propose an explore then commit type algorithm for the same.

We study the matching markets problem where the rewards are weight distorted (via CPT) and corrupted simultaneously. We consider two cases where the total corruption budget is known and unknown. In both cases we modify the \texttt{CPT-ETGS} algorithm and obtain sub-linear regret guarantees. In particular, we characterize the (regret) cost that comes from simultaneous control of distorted and corrupted rewards.

\subsection{Summary of Contributions}
We leverage CPT based risk sensitive measure to capture human centric decision making in matching markets. We use the state of the art learning algorithm for matching markets from \cite{kong2023player} and incorporate CPT based weight distortion. The resulting algorithm, \texttt{CPT-ETGS} is shown to achieve a (player optimal) regret upper bound of $\mathcal{O}(K\log T \left(\frac{1}{\Delta}\right)^{2/\alpha})$, where $\Delta$ is the minimum preference gap between the first $N+1$ arms across all players and $\alpha$ is the coefficient of H\"older continuity for the (non-linear) weight distortion function. \texttt{CPT-ETGS} is also an explore and commit type algorithm, however the upper and lower confidence bounds are carefully defined to incorporate the weight distortion of the rewards.

In Section~\ref{sec:improved}, we consider a simple problem instance and obtain a regret lower bound (roughly) $\Omega \left (K\log T \left(\frac{1}{\Delta} \right)^{\frac{2}{\alpha} -1}\right)$, which implies the regret of \texttt{CPT-ETGS} is sub-optimal (since $\frac{1}{\Delta^{2/\alpha -1}}$ may be much smaller than $\frac{1}{\Delta^{2/\alpha}}$). We identify that the sub-optimality stems from \emph{excess} exploration in \texttt{CPT-ETGS}. 

In the modified algorithm namely Improved \texttt{CPT-ETGS}, players adaptively switch between exploration and exploitation and promptly eliminate sub-optimal arm so the active set of arms on which we will be exploring will not lead to unnecessary regret. We prove that Improved \texttt{CPT-ETGS} achieves a regret upper bound of $\mathcal{O}\!\left(\;K\log T \left(\frac{1}{\Delta} \right)^{\frac{2}{\alpha} -1} +N^2 \log T \, \frac{1}{\Delta^{2/\alpha}} \right)$. In the scenario where $K$ is much larger than $N$ ($K \gg N$), i.e., in large markets, the first term dominates and the regret of Improved \texttt{CPT-ETGS} matches to that of the lower bound. This implies optimality of Improved \texttt{CPT-ETGS}.

Furthermore, in Section~\ref{sec:adversarial}, we consider adversarial markets, the practical setup where the rewards are distorted as well as corrupted. In this setup, players may not be able to learn their true preferences owing to the presence of corruption. We address such adversarial market with known as well as unknown corruption budget (denoted by $C$). Given the corruption budget, we may bound the changes in distorted reward using the H\"older continuity property of the weight distortion function. In this case we widen the confidence interval while calculating preferences. If the corruption budget is unknown, we use a multi-layer \texttt{CPT-ETGS}. In both the cases we change the confidence intervals, and regret definitions to incorporate human preferences. If $C$ is known, we achieve a regret upper bound of $\mathcal{O}\left(
\frac{K\log T}{\Delta^{2/\alpha}}
+ K\left(\frac{C}{\Delta}\right)^{1/\alpha}\right)$. On the other hand, if $C$ is not known, we obtain an upper bound of $\mathcal{O}\!\left(
KC\left(\log^2T\right)
\left(\frac{1}{\Delta}\right)^{\frac{2}{\alpha}}
+
K(\log T+C)
(\log T)^{1+\frac{1}{\alpha}}
\left(\frac{d}{\Delta}\right)^{\frac{1}{\alpha}}
\right)$ where $d$ is a problem dependent parameter,
and the difference between these two characterizes the difficulty of not knowing the corruption budget.

\subsubsection{Technical Challenges}
The main technical challenge is to integrate CPT based weight distortion in the multi-agent framework with competition. In particular, we design new confidence intervals (in such a setup) such that the estimated preference is close to the true preference.  We achieve this by an application of Dvoretzky-Kiefer-Wolfowitz (DKW) inequality which gives finite-sample exponential concentration of the empirical distribution. Subsequently, we define a \emph{good} event for \texttt{CPT-ETGS} and Improved \texttt{CPT-ETGS}. Note that sample mean estimation with CPT distortion uses order statistics (see \cite{DBLP:journals/corr/GopalanAFM16}), and with multiple agents and unknown preferences, such analysis becomes quite non-trivial. Furthermore, in the case of corrupted rewards, we leverage the H\"older continuity property of the CPT weight distortion function and modify the upper and lower confidence bounds. This leads to a bound on the maximum deviation between distorted reward estimates computed with corrupted and uncorrupted observations. In particular, with unknown corruption budget, tracking the corrupted and distorted reward simultaneously turns out to be quite involved.

\begin{figure}[t!]                                   
      \centering  
      \begin{subfigure}[]{0.45\textwidth}         
          \centering                            
          \includegraphics[height = 4.5cm]{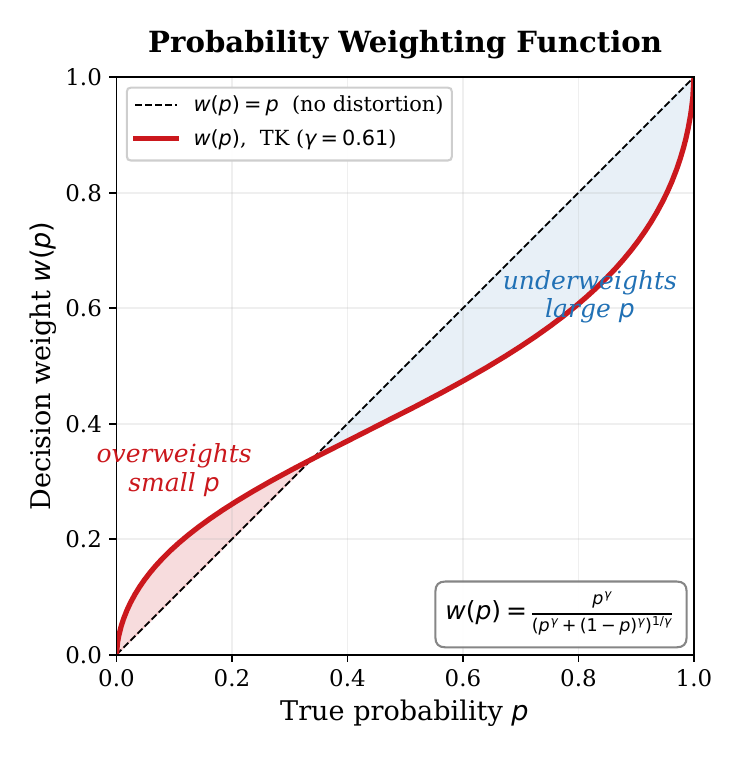}  \caption{Tversky--Kahneman probability weighting function $w(p)$.            Small probabilities are \emph{overweighted} ($w(p) > p$) and
          large probabilities are \emph{underweighted} ($w(p) < p$).}
          \label{fig:weight-fn}
      \end{subfigure}
      \hfill
      \begin{subfigure}[]{0.45\textwidth}
          \centering
          \includegraphics[height = 4cm]{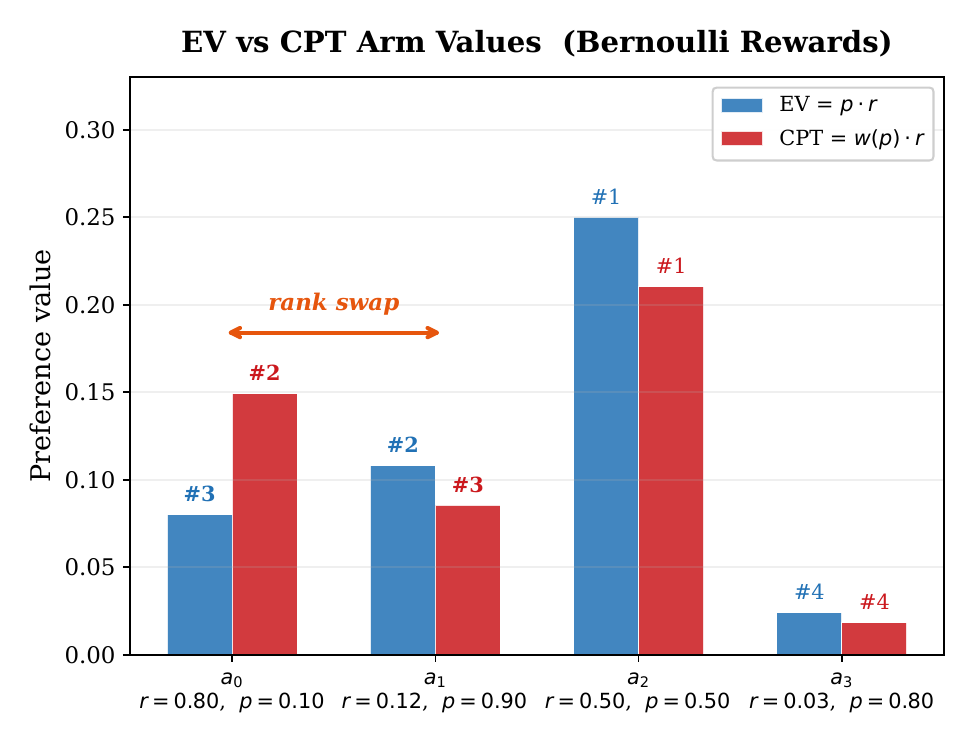}
          \caption{EV vs.\ CPT arm values for a representative player.
          Arms $a_0$ and $a_1$ swap ranks: $a_0$ (rare high reward,
          $p{=}0.10$) is boosted under CPT while $a_1$ (frequent low
          reward, $p{=}0.90$) is demoted.}
          \label{fig:bar-chart}
      \end{subfigure}
      \caption{\textbf{CPT distortion leads to different preference
      orders.} Each arm follows a Bernoulli distribution: reward~$r$ with
      probability~$p$, and~$0$ otherwise. Under expected value, a player
      ranks arms by $\mu_j = p_j r_j$. Under CPT, the same samples
      yield rankings based on $\nu_j = w(p_j) r_j$. Since
      $w(\cdot)$ overweights small probabilities~(a), arms with rare but
      large rewards can overtake arms with frequent but modest rewards~(b),
      producing a \emph{different} preference order from the same data. Hence, ETGS and \texttt{CPT-ETGS} converge to different stable matching.}
      \label{fig:preference-reversal}
  \end{figure}

\section{Formulation}
\label{sec:formulation}
We formally introduce the problem setting below. For any positive integer $n$, we use $[n]$ to denote $\{1,2,\ldots,n\}$.

\subsection{Risk Neutral (Undistorted) Matching Market Model}
We study a two-sided market comprising $N$ players $\mathcal{P}=\{p_1,\dots,p_N\}$ and
  $K\!\ge\!N$ arms $\mathcal{A}=\{a_1,\dots,a_K\}$. Each player $p_i$ is characterized by unknown
  mean rewards $\bar{\mu}_{i,j}$ over arms, while each arm $a_j$ possesses a fixed, publicly known
  strict ranking $(\pi_{j,i})_{i\in[N]}$ over the players with all values distinct.

  The interaction unfolds over $T$ rounds.  In round $t$, every player $p_i$ selects an arm
  $A_i(t)\in\mathcal{A}$ to propose to.  Each arm accepts its highest-ranked proposer under~$\pi_j$
   and rejects the remaining players.  Denoting by $\bar{A}_i(t)$ the arm that accepts~$p_i$ (with
  $\bar{A}_i(t)=\emptyset$ if rejected), the accepted player observes a stochastic reward
  $X_i(t)\sim\mathcal{F}_{i,\bar{A}_i(t)}$ with mean $\bar{\mu}_{i,\bar{A}_i(t)}$, while a rejected
   player receives $X_i(t)=0$.  The global matching at round $t$ is
  $\bar{A}(t)=\{(i,\bar{A}_i(t)):i\in[N]\}$.

  A matching is \emph{stable} if no player--arm pair $(p_i,a_j)$ would mutually benefit from
  deviating, i.e., there is no pair satisfying $\bar{\mu}_{i,j}>\bar{\mu}_{i,\bar{A}_i(t)}$ and
  $\pi_{j,i}>\pi_{j,\bar{A}_j^{-1}(t)}$.  Among all stable matchings~$M$, the \emph{player-optimal
  stable matching} $m^*$ is the unique element for which $\bar{\mu}_{i,m^*_i}\ge\bar{\mu}_{i,m_i}$
  for every $m\in M$ and $i\in[N]$; its existence is guaranteed by the Gale-Shapley
  deferred-acceptance procedure~\cite{gale1962college}.

\begin{assumption}
\label{ass:bounded}
The reward $X_i(t)$ is bounded almost surely, i.e., $|X_i(t)| \leq M$ for some constant $M > 0$.
\end{assumption}

\subsection{CPT based Distorted Utility Model}

Instead of evaluating actions using the standard expected reward, we adopt a
distorted utility model based on cumulative prospect theory (CPT) \cite{tversky1992advances}. In this
framework, players evaluate stochastic outcomes through a nonlinear
transformation of probabilities.

Let $Y_{i,j}$ denote the random reward obtained by player $p_i$ when matched
with arm $a_j$. Let $w : [0,1] \rightarrow [0,1]$ be a probability weighting
function satisfying $w(0)=0$ and $w(1)=1$. The distorted utility of arm $a_j$
for player $p_i$ is defined as
\begin{equation}
\label{eq:cpt}
\mu_{i,j}
=
\int_{0}^{\infty} w\big(\mathbb{P}(Y_{i,j} > z)\big) \, dz
-
\int_{0}^{\infty} w\big(\mathbb{P}(-Y_{i,j} > z)\big) \, dz .
\end{equation}

The above expression represents the CPT value of the reward distribution by integrating the distorted tail probabilities of gains and losses. For analytical convenience, we decompose the distorted utility into its positive and negative components.
Let
$
\mu_{i,j} = \mu^{+}_{i,j} - \mu^{-}_{i,j},
$
where the gain component and loss component are defined respectively as
$
\mu^{+}_{i,j} = \int_{0}^{\infty} w\big(\mathbb{P}(Y_{i,j} > z)\big)\,dz
$
and
$
\mu^{-}_{i,j} = \int_{0}^{\infty} w\big(\mathbb{P}(-Y_{i,j} > z)\big)\,dz.
$

The term $\mu^{+}_{i,j}$ captures the distorted contribution of gains, while $\mu^{-}_{i,j}$ captures the distorted contribution of losses~\cite{kolla2016bandit,DBLP:journals/corr/GopalanAFM16}. The probability weighting function $w(\cdot)$ distorts the tail probabilities, reflecting the behavioral bias described by cumulative prospect theory \cite{kolla2016bandit,DBLP:journals/corr/GopalanAFM16}. In particular, small probabilities may be overweighted while moderate or large probabilities may be underweighted depending on the shape of $w$.

Note that when the weighting function is the identity function
$w(p) = p$, the distorted utility reduces to the standard expected reward,
i.e., $\mu_{i,j} = \Bar{\mu}_{i,j} =  \mathbb{E}[Y_{i,j} ].$ Thus the distorted utility formulation
generalizes classical expected-value based evaluation.

Throughout the paper, we use $\mu_{i,j}$ to denote the CPT-distorted reward defined in equation~\ref{eq:cpt}, rather than the standard expected reward $\bar{\mu}_{i,j}$.

\begin{assumption} (H\"older continuity)
\label{ass:holder}
Let $w:[0,1]\rightarrow[0,1]$ denote the probability weighting function.
We assume that $w(0)=0$ and $w(1)=1$. Furthermore, $w(\cdot)$ satisfies a
H\"older continuity condition: there exist constants $L>0$ and
$\alpha \in (0,1]$ such that for all $x,y \in [0,1]$,
\begin{equation}
    \sup_{x \neq y}
\frac{|w(x) - w(y)|}{|x-y|^{\alpha}} \le L .
\end{equation}
\end{assumption}
Similar assumptions featured in previous works \cite{DBLP:journals/corr/GopalanAFM16,pmlr-v258-tatli25a,kolla2016bandit} with CPT risk.
\subsection{CPT Regret}
Let $m^*$ denote the player-optimal stable matching.
The goal of each player is to learn the arm assigned to it in $m^*$.
The player-optimal stable regret of player $p_i$ after $T$ rounds
is defined as
$
\texttt{CPT-Reg}_i(T)=T \mu_{i,m_i^*}-\sum_{j \in [K]} T_{i,j}(T)\,\mu_{i,j},
$
and the expected regret can be written as
\begin{align}
    \mathbb{E}[\texttt{CPT-Reg}_i(T)] =
\sum_{j \in [K],\, j \neq m_i^*}
\mathbb{E}[T_{i,j}(T)]\,\Delta_{i,m_i^*,j} .
\end{align}
where
$
T_{i,j}(T)
=
\sum_{t=1}^{T}
\mathbf{1}\{\bar{A}_i(t) = a_j\}
$
is the number of times arm $j$ is matched by player $i$ up to time $T$ and $\Delta_{i,m_i^*,j}$ is defined below.

\emph{Preference Gap}~\cite{kong2023player}: Consider a player $p_i$ and two distinct arms $a_j$ and $a_{j'}$.
The preference gap of player $p_i$ between these two arms is defined as
$\Delta_{i,j,j'} = |\mu_{i,j} - \mu_{i,j'}|.$
Let $\rho_i$ denote the preference ordering of player $p_i$ over the arms,
and let $\rho_{i,j}$ represent the arm ranked at position $j$ in this ordering,
for $j \in [K]$.
We define
$\Delta =
\min_{i \in [N];\, j,j' \in [N+1];\, j \neq j'}
\Delta_{i,\rho_{i,j},\rho_{i,j'}}$
to be the smallest preference gap taken over all players and over
all pairs of arms among their first $N+1$ ranked arms. Since all
preference values are assumed to be distinct, this quantity is
non-negative.
Finally, for each player $p_i$, define
$\Delta_{i,\max} = \mu_{i,m_i^*}$
which corresponds to the maximum player-optimal stable regret that
player $p_i$ may incur across all rounds.

\section{Human Centric Matching Markets with CPT Regret}
\subsection{Distorted Preference Estimation}

To estimate player preferences under cumulative prospect theory (CPT),
we use a distorted utility estimator based on the empirical reward
distribution.

\begin{algorithm}[t]
\caption{CPT Distorted Preference Estimation \cite{kolla2016bandit}}
\label{alg:distorted_estimator}
\begin{algorithmic}[1]
\State \textbf{Input:} rewards $X_{i,j,1},\dots,X_{i,j,l}$, weight function $w(\cdot)$
\State Sort samples: $X_{i,j,(1)} \le \dots \le X_{i,j,(l)}$ and let $l_b$ be the number of negative samples
\State $\hat{\mu}^{+}_{i,j}=\sum_{k=l_b+1}^{l}X_{i,j,(k)}\big[w(\frac{l+1-k}{l})-w(\frac{l-k}{l})\big]$
\State $\hat{\mu}^{-}_{i,j}=\sum_{k=1}^{l_b}X_{i,j,(k)}\big[w(\frac{k-1}{l})-w(\frac{k}{l})\big]$
\State \textbf{Return:} $\hat{\mu}_{i,j}=\hat{\mu}^{+}_{i,j}-\hat{\mu}^{-}_{i,j}$
\end{algorithmic}
\end{algorithm}
Algorithm~\ref{alg:distorted_estimator} computes an empirical estimate
of the distorted utility using the observed rewards. 
The observed rewards are first sorted so that the ordered samples
correspond to empirical quantiles of the reward distribution.
This quantile representation allows the distorted utility, which is
defined through distorted tail probabilities, to be approximated using
order statistics of the empirical distribution.

The distortion function $w(\cdot)$ is then applied to the probability
levels associated with these quantiles. The terms
$w\!\left(\frac{l+1-k}{l}\right)-w\!\left(\frac{l-k}{l}\right)$
represent the distorted probability mass assigned to the $k$-th
order statistic.

Since CPT evaluates gains and losses separately, the estimator
computes the contributions of positive and negative rewards
independently. The distorted gain component aggregates positive
samples, while the distorted loss component accounts for negative
samples. Combining these components yields the empirical estimate
$\hat{\mu}_{i,j}$. This quantile-based estimator follows the CPT-value estimation
scheme introduced in \cite{pmlr-v48-la16} and later adapted for bandit settings in \cite{kolla2016bandit}.
\subsection{Algorithm: \texttt{CPT-ETGS}}
In this section, we consider the problem of finding the player-optimal stable matching in human-centric matching markets, where players evaluate outcomes based on behavioral preferences rather than purely expected rewards.
To address this, we propose a CPT-aware algorithm that builds on the Explore Then Gale Shapley (ETGS) framework \cite{kong2023player,kong2024improved}. The key distinction between Algorithm \ref{alg:cpt_etgs} and the  ETGS algorithm of \cite{kong2023player} lies in the estimation of player preferences over arms.  This modification enables the algorithm to learn preference values aligned with CPT evaluations using Algorithm \ref{alg:distorted_estimator} while preserving the ETGS framework.

\begin{algorithm}[t]
\caption{\texttt{CPT-ETGS} (Player $p_i$)}
\label{alg:cpt_etgs}
\begin{algorithmic}[1]

\Require player set $P$, arm set $A$, horizon $T$
\State Initialize $\hat{\mu}_{i,j}=0$, $T_{i,j}=0$

\State \textbf{Phase 1: Index discovery}
\State \textbf{Phase 2: Preference learning}

\For{sub-phase $\ell =1,2,\dots$}
\For{$2^\ell$ rounds}
\State select arm via round-robin based on Index
\State observe reward, increment $T_{i,j}$ and update $\hat{\mu}_{i,j}$ using Algorithm~\ref{alg:distorted_estimator}

\EndFor

\State compute W-UCB and W-LCB for all arms

\If{top $N$ arms identified}
\State record ranking $\sigma_i$
\EndIf

\If{all players matched}
\State terminate exploration
\EndIf

\EndFor

\State \textbf{Phase 3: Stable matching} via decentralized GS algorithm
\end{algorithmic}
\end{algorithm}

\texttt{CPT-ETGS} proceeds in three phases. In the first phase, each player obtains a unique index that allows players to coordinate their exploration and avoid collisions, in the second phase players estimate their preferences over arms by exploring with the help of the index obtained in the first phase, and in the third phase using the preference order obtained in second phase players find out the stable matching and then would focus on that arm for the remaining rounds. The details of the first and third phase are mentioned in the Appendix \ref{app:phase_1_3}. 

The second phase of the algorithm is divided into multiple subphases $l = 1,2,3,...$ and each subphase has a length of $2^l+1$ consisting of an exploration stage of $2^l$ rounds followed by a monitoring round. In the exploration rounds, the players choose arms in a round robin fashion to avoid collision and to make estimates of their preference over arms. In the monitoring round, all the players estimate their CPT-distorted preference values over all arms using Algorithm \ref{alg:distorted_estimator} adapted from \cite{kolla2016bandit}, and construct confidence intervals for each arm.

The confidence bounds are defined as
\begin{align}
    \text{W-UCB}_{i,j} = \hat{\mu}_{i,j} + LM\left(\frac{3\log T}{2T_{i,j}(t)}\right)^{\alpha/2},
\text{W-LCB}_{i,j} = \hat{\mu}_{i,j} - LM\left(\frac{3\log T}{2T_{i,j}(t)}\right)^{\alpha/2} 
\label{eq:wucb}
\end{align}
We set $\text{W-UCB}_{i,j}$ to $\infty$ and $\text{W-LCB}_{i,j}$
to $-\infty$ when $T_{i,j}=0$. Once the confidence sets for two arms $a_j, a_{j'}$
are disjoint, i.e., $\text{W-LCB}_{i,j} > \text{W-UCB}_{i,j'}$ or
$\text{W-LCB}_{i,j'} > \text{W-UCB}_{i,j}$, $p_i$ can determine its preference over these arms. When a player is able to
identify the top $N$ arms in its ranking, then in the monitoring round the player propose to the arm with the unique index obtained in the first phase, and the players who have not yet determined their preferences remain inactive. When all the players are matched with an arm, we can determine that all players have found their preference order over arms and can proceed to the third phase.
We next derive an upper bound on the player-optimal stable pseudo-regret incurred by each player when using our algorithm. The proof is deferred to Appendix \ref{app:proof_cpt_etgs}.
\begin{theorem}
\label{th:cpt_etgs}
The player optimal stable regret incurred by each $p_i \in \mathcal{N}$ under Algorithm~\ref{alg:cpt_etgs} satisfies
\begin{align*}
\mathbb{E}[\texttt{CPT-Reg}_i(T)]
&\le c_0
\left(
\frac{3K\log T}{\Delta^{2/\alpha}}
\left(4LM\right)^{2/\alpha}
+ N^2
+ 2NK
\right)\Delta_{i,\max}.
\end{align*}
where $c_0$ is a universal positive constant.
\end{theorem}
\begin{remark}
    Observe that the CPT-regret of player $i$ is $O\!\left(
K\log T \left(\frac{LM}{\Delta}\right)^{2/\alpha}
\right)$. In Section~\ref{sec:lower_bound}, we show that the dependence on $K$ and $T$ is optimal. However the dependence on the gap $\Delta$ is sub-optimal. In the subsequent section, we modify \texttt{CPT-ETGS} to achieve such optimal regret.
\end{remark}
\begin{remark}
    When $\alpha=1$, the CPT-regret reduces to the usual risk neutral regret, matching that of ETGS \cite{kong2023player}. Even for a (single agent) simple bandit problem, such equivalence has been observed in \cite{kolla2016bandit,DBLP:journals/corr/GopalanAFM16}.
\end{remark}
\section{Improved \texttt{CPT-ETGS} and Optimal Regret}
\label{sec:improved}

\subsection{Regret Lower Bound}
\label{sec:lower_bound}
In this section, we obtain a lower bound on player-optimal stable regret. We adhere to the classical asymptotic regret lower bound framework of Lai and Robbins \cite{LAI19854}. These lower bounds are also studied in \cite{pmlr-v130-sankararaman21a} in the context of (risk neutral) markets and bandits.

Consider an instance where all arms share the same preference ordering over players, $p_1 \succ p_2 \succ \dots \succ p_N$. This is called the global ranking \cite{liu2021bandit} or the serial dictatorship \cite{pmlr-v130-sankararaman21a} model.  We characterize the regret lower bound based on this instance.

\begin{lemma}[Lower bound on player-optimal stable regret]
\label{lemma:lower_bound}
There exists a problem instance with serial dictatorship such that for any learning algorithm, the player-optimal stable regret of player $p_1$ satisfies
\[
\liminf_{T \to \infty}
\frac{\mathbb{E}[\texttt{CPT-Reg}_1(T)]}{\log T}
\ge
\sum_{j \in [K], j \neq m_1^*} \frac{1}{4}(LM)^{2/\alpha} \left( \frac{1}{\Delta_{1,m_1^*,j}} \right)^{\frac{2}{\alpha}-1}.
\]
where $\Delta_{1,m_1^*,j}$ is the preference gap between the best arm $m^*_1$ and the $j$-th arm.
\end{lemma}
\begin{remark}
    The above result roughly implies that the (minimax) lower bound for CPT distorted matching markets has an improved dependence on preference gap compared to the regret of \texttt{CPT-ETGS}.
\end{remark}

\emph{Intuition:} It is easy to see that under serial dictatorship, the stable matching between the players and arms is unique. Player ranked $1$ ($p_1$) prefers the arm $m_1^* = \arg\max_{j \in [K]} \mu_{1j} $
the most, where $\mu_{ij}$ is the CPT distorted mean of player $i$ pulling arm $j$. Since player $p_1$ is the highest ranked player, it will not face any collision and the learning problem for $p_1$ reduces to the classical multi-armed bandit problem (with CPT distorted rewards) where a player is given $K$ actions (arms) and must learn to choose the best action over time. Detailed proof of the lemma is provided in the Appendix~\ref{app:lower_bound}.

\begin{algorithm}[t!]
\caption{Improved \texttt{CPT-ETGS} (player $p_i$)}
\label{alg:improved}
\begin{algorithmic}[1]
\State \textbf{Input:} number of players $N$, number of arms $K$, horizon $T$
\State \textbf{Initialize:} $\hat{\mu}_{i,j}=0$, $T_{i,j}=0$, $\text{W-UCB}_{i,j}=\infty$, $\text{W-LCB}_{i,j}=-\infty$ for all $j$
\State $\mathcal{D}_i \gets \emptyset$, $\mathcal{A}_i \gets \mathcal{K}$, $E_i \gets$ True

\For{round $t = 1,2,\dots,T$}

    \State Select an arm $A_i(t)$ from $\mathcal{A}_i$ using a round-robin schedule

    \State Update $\hat{\mu}_{i,j}$ using Algorithm~\ref{alg:distorted_estimator}, increment $T_{i,j}$ 

    \State Remove suboptimal arms from $\mathcal{A}_i$ using the confidence bounds if $|\mathcal{A}_i| > N$

    \If{one arm is confidently better than all others}
        \State Commit to that arm: set $\mathcal{A}_i=\{j\}$ and $E_i=\text{False}$
    \EndIf

    \For{each other player $p_{i'}$ already committed to arm $A_{i'}$}
        \If{arm $A_{i'}$ prefers $p_{i'}$ over $p_i$}
            \State Add $A_{i'}$ to deletion set $\mathcal{D}_i$ and remove it from $\mathcal{A}_i$
            \If{$p_i$ had committed to that arm}
                \State Resume exploration: set $E_i=\text{True}$
            \EndIf
        \EndIf
    \EndFor

\EndFor
\end{algorithmic}
\end{algorithm}

\subsection{Improved \texttt{CPT-ETGS} Algorithm}

 To improve the dependence on K in the regret obtained in Theorem~\ref{th:cpt_etgs} we build on the adaptive ETGS framework of~\cite{kong2024improved}, modifying it to operate with 
  CPT-distorted confidence bounds. The key idea is to interleave exploration with arm elimination
  rather than exploring all $K$ arms uniformly (as in \texttt{CPT-ETGS}), players progressively    
  discard arms that are provably suboptimal under the distorted reward estimates, thereby reducing
  unnecessary exploration.

  Each player $p_i$ maintains a candidate arm set $\mathcal{A}_i$ (initially $[K]$), a rejected
  arm set $D_i$ (initially $\emptyset$), and an exploration flag $E_i$ (initially True) indicating
  whether the player is still exploring or has committed to a single arm.

  At each round $t$, player $p_i$ updates its CPT-distorted estimate $\hat{\mu}_{i,j}(t)$ via
  Algorithm~\ref{alg:distorted_estimator} and constructs the confidence intervals based on W-UCB and W-LCB as defined in equation~\ref{eq:wucb}.
  An arm whose W-UCB falls below the W-LCB of any other arm is eliminated from $\mathcal{A}_i$.
  This elimination is the central difference from \texttt{CPT-ETGS}, it ensures that the number
  of samples wasted on clearly inferior arms scales with the per-arm gap $\Delta_{i,m_i^*,j}$
  rather than the global minimum gap $\Delta$, leading to improved regret.

  The Algorithm~\ref{alg:improved} operates as an online counterpart of the
  Gale-Shapley procedure~\cite{gale1962college}. While $E_i = \text{True}$, player $p_i$
  explores arms from $\mathcal{A}_i$ via round-robin and eliminates suboptimal arms using the
  confidence bounds. Once a single arm dominates all others in $\mathcal{A}_i$, the player sets
  $E_i = \text{False}$ and commits to that arm. If a committed arm is subsequently claimed by a
  higher-priority player (according to the arm's ranking $\pi_j$), the arm is moved to $D_i$,
  the candidate set is reset to $\mathcal{A}_i = [K] \setminus D_i$, and $E_i$ is restored to
  True so that exploration resumes over the remaining arms. This interaction process mirrors
  the offline Gale-Shapley algorithm, ensuring convergence to the player-optimal stable matching.
  We refer to~\cite{kong2024improved} for details on the communication protocol and collision
  handling (Lemma~B.2 therein guarantees that each arm is sampled at least once every $2N$ rounds
  despite collisions).

\begin{theorem}
\label{th:improved}
The player-optimal stable regret for each player $p_i$ satisfies
\begin{align*}
    \mathbb{E}[\texttt{CPT-Reg}_i(T)] \le c_1\!\left(\sum_{j \in [K], j \neq m_i^*}\log T \,\frac{(LM)^{2/\alpha}}{\Delta_{i,m_i^*,j}^{2/\alpha - 1}} + N^2 \log T \, \frac{(LM)^{2/\alpha}}{\Delta^{2/\alpha}} \cdot \Delta_{i,\max}
  \right)
\end{align*}
where $c_1$ is a universal positive constant.
\end{theorem}
The proof is deferred to Appendix~\ref{app:improved}.
\begin{remark}
    When $K \gg N$, the regret for Improved \texttt{CPT-ETGS} matches to that of the lower bound obtained in Lemma~\ref{lemma:lower_bound}, implying optimality of the algorithm.
\end{remark}

\section{Corrupted Matching Markets with CPT Regret}
\label{sec:adversarial}
\subsection{Adversarial Corruption Model}

 We extend the CPT-matching framework to accommodate adversarial reward corruption, following the
 model of \cite{wubandit}. At each round $t$,  a true reward
  $r^S_{i,j}(t) \sim \mathcal{F}_{i,j}$ is generated for every player-arm pair who get matched. An adversary, with access to the
   full history of rewards and matching decisions up to round $t$, then perturbs these values and
  presents a corrupted reward $r_{i,j}(t)$ in their place. A matched player $p_i$ therefore
  observes $X_i(t) = r_{i,A_i(t)}(t)$ rather than the true stochastic reward.
\begin{assumption}
\label{ass:bound_C}
We assume that the total corruption experienced by player $p_i$ over the time horizon $T$ is bounded by
$
\sum_{t=1}^{T} \max_{j \in [K]} \left| r_{i,j}(t) - r^S_{i,j}(t) \right| \le C .
$
\end{assumption}
\begin{assumption}
\label{ass:positive_rewards}
     The stochastic rewards are non-negative and bounded by a constant $M$ almost surely, i.e.,
$0 \le r^S_{i,j}(t) \le M \quad \forall i,j,t,$
and the possibly corrupted rewards are non-negative, i.e.,
$r_{i,j}(t) \ge 0 \quad \forall i,j,t.$
\end{assumption}
\begin{remark}
We assume corrupted rewards are non-negative since the stochastic rewards are also non-negative. 
If a negative reward were observed, the player could immediately detect that corruption has occurred. 
Moreover, any negative corrupted reward can be safely replaced by $0$ without affecting the analysis.
\end{remark}

\subsection{Algorithm with known corruption budget}
We now consider the setting where the total corruption budget $C$ is known in advance. In this case, we modify the \texttt{CPT-ETGS} algorithm to make the learning process robust to adversarially corrupted rewards, as in Algorithm 1 of \cite{wubandit}.  We play Algorithm~\ref{alg:cpt_etgs} and to account for the corruption while players estimate their preferences over arms, we enlarge the confidence intervals as
$
\text{W-UCB}_{i,j} = \hat{\mu}_{i,j} + LM\!\left(\frac{3\log T}{2T_{i,j}(t)}\right)^{\!\alpha/2} + \frac{CL}{T_{i,j}^{\alpha}}$, $
\text{W-LCB}_{i,j} = \hat{\mu}_{i,j} - LM\!\left(\frac{3\log T}{2T_{i,j}(t)}\right)^{\!\alpha/2} - \frac{CL}{T_{i,j}^{\alpha}}$.

The rest of the three-phases: index discovery, preference learning via round-robin exploration, and decentralized Gale-Shapley matching remains unchanged. This modification ensures that the true CPT-distorted preference value stays within the confidence interval despite adversarial perturbations.

\begin{theorem}
\label{th:known}
With a (known) corruption budget of $C$, the CPT regret of
$p_i \in \mathcal{N}$ satisfies $\mathbb{E}[\texttt{CPT-Reg}_i(T)]=\mathcal{O}\left(
\frac{K\log T}{\Delta^{2/\alpha}}
+ K\left(\frac{C}{\Delta}\right)^{1/\alpha}\right)$.
\end{theorem}
When $\alpha=1$, the CPT regret matches to that of \cite{wubandit} where undistorted rewards are considered. The proof is deferred to Appendix~\ref{app:proof_known}.
\subsection{Algorithm with Unknown Corruption}

\begin{algorithm}[t]
\caption{CPT-Multi-layer ETGS Race (player $p_i$)}
\label{alg:etgs_unknown}
\begin{algorithmic}[1]
\State \textbf{Input:} $\mathcal{N}, \mathcal{K}, T$, sub-phase length $d$
\State \textbf{Initialize:} $\hat{\mu}^{\ell}_{i,j}=0$, $T^{\ell}_{i,j}=0$ for $j\in[K], \ell\in[\log T]$
\State \textbf{Phase 1:} Estimate player index using ETGS

\For{sub-phase $k = 1,2,\dots$}

\State Explore arms for $d$ rounds using the layer-$\ell$ round-robin rule and update $\hat{\mu}^{\ell}_{i,j},T^{\ell}_{i,j}$

\State Compute confidence bounds $W\text{-}\mathrm{UCB}^{\ell}_{i,j}$ and $W\text{-}\mathrm{LCB}^{\ell}_{i,j}$ for all $j\in[K]$

\State Apply \textsc{Monitoring} to obtain ranking $\sigma_i^{\ell}$

\If{$N$ arms are identified}
\State Compute the stable arm via decentralized Gale–Shapley and assign the matching to all layers $\ell' \le \ell$
\EndIf

\If{$p_i$ is the leader}
\State Sample layer $\ell$ with probability $2^{-\ell}$, with remaining probability sample $\ell=1$ and broadcast using $c_k=\lfloor \ell \rfloor$ arm pulls
\EndIf

\EndFor
\end{algorithmic}
\end{algorithm}

We now consider the setting where the total corruption level $C$ is unknown. 
 The algorithm retains the three-phase structure of \texttt{CPT-ETGS} (index discovery, preference
   learning, stable matching). 
  The core challenge is that without knowing $C$, we cannot calibrate the corruption term in the
  confidence intervals. Following~\cite{wubandit}, we address this by running $\log T$ parallel
  layers of the learning algorithm, where layer $\ell$ is activated with probability
  $2^{-\ell}$. Layers that are selected less frequently accumulate fewer corrupted samples in
  expectation.
  Exploration within each layer proceeds in fixed-length sub-phases of $d$ rounds, with
  observations updating only the active layer's statistics. The CPT-distorted estimates are
  computed via Algorithm~\ref{alg:distorted_estimator}.

  For layer $\ell$, we construct confidence intervals as
  \[
  \text{W-UCB}^{\ell}_{i,j}
  =
  \hat{\mu}^{\ell}_{i,j}
  +
  LM\left(\frac{3\log T}{2T^{\ell}_{i,j}(t)}\right)^{\alpha/2}
  +
  \frac{2dL\log T}{(T^{\ell}_{i,j}(t))^\alpha},
  \]
  \[
  \text{W-LCB}^{\ell}_{i,j}
  =
  \hat{\mu}^{\ell}_{i,j}
  -
  LM\left(\frac{3\log T}{2T^{\ell}_{i,j}(t)}\right)^{\alpha/2}
  -
  \frac{2dL\log T}{(T^{\ell}_{i,j}(t))^\alpha}.
  \]
  The first correction term handles stochastic estimation error (as in \texttt{CPT-ETGS}), while
  the second accounts for corruption.
  Layer synchronization across players is achieved through a leader-based protocol
  (see~\cite{wubandit} for details): the leader selects a layer at the end of each sub-phase. Once a layer $\ell$ with sampling probability below $1/C$ confidently determines
   the preference ordering, it computes the stable matching $\sigma^{\ell}$ via decentralized
  Gale-Shapley. By Lemma~\ref{lemma:robust}, this layer's estimates are reliable, so its matching
  is propagated to all faster layers $\ell' \leq \ell$.

\begin{theorem}
\label{th:unknown}
The CPT-Multi-layer ETGS race algorithm achieves the player-optimal stable regret upper bound of each player $p_i$, $\mathbb{E}[\texttt{CPT-Reg}_i(T)]$ as
\begin{align*}
\mathcal{O}\!\left(
\left[
KC\left(\log^2T\right)
\left(\frac{LM}{\Delta}\right)^{\frac{2}{\alpha}}
+
K(\log T+C)
(\log T)^{1+\frac{1}{\alpha}}
\left(\frac{dL}{\Delta}\right)^{\frac{1}{\alpha}}
\right]
\Delta_{i,\max}
\right).
\end{align*}
\end{theorem}
Observe that although worse compared to Theorem~\ref{th:known}, we still obtain logarithmic regret. The proof is deferred to Appendix~\ref{app:proof_unknown}.

\section{Experiments}
\label{sec:experiments}

We validate our theory on synthetic matching markets. We study three questions:
(i)  whether ignoring CPT distortion might lead to incorrect stable matchings and whether \texttt{CPT-ETGS}
achieves sub-linear (logarithmic) regret as predicted by
Theorem~\ref{th:cpt_etgs} (ii) whether adaptive elimination in Improved
\texttt{CPT-ETGS} improves performance when many arms are clearly sub-optimal
(Theorem~\ref{th:improved}); and (iii) whether the robust algorithms of
Section~\ref{sec:adversarial} peform better than \texttt{CPT-ETGS} and remain sub-linear under adversarial corruption with
known and unknown budgets (Theorems~\ref{th:known}--\ref{th:unknown}).
\paragraph{Weight function and setup.}
We use the Lipschitz CPT
weight
\[
w(p)=p+c\,p(1-p)(1-2p), \quad c=1.5,
\]
an inverse-$S$ distortion satisfying $w(0)=0$ and $w(1)=1$. It is
$\alpha=1$ H\"older with constant $L=\sup_p |w'(p)|=2.5$. Since rewards
lie in $[0,1]$, $M=1$.

\begin{figure}[t]
  \centering
  \begin{subfigure}[t]{0.32\linewidth}
    \centering
    \includegraphics[width=\linewidth]{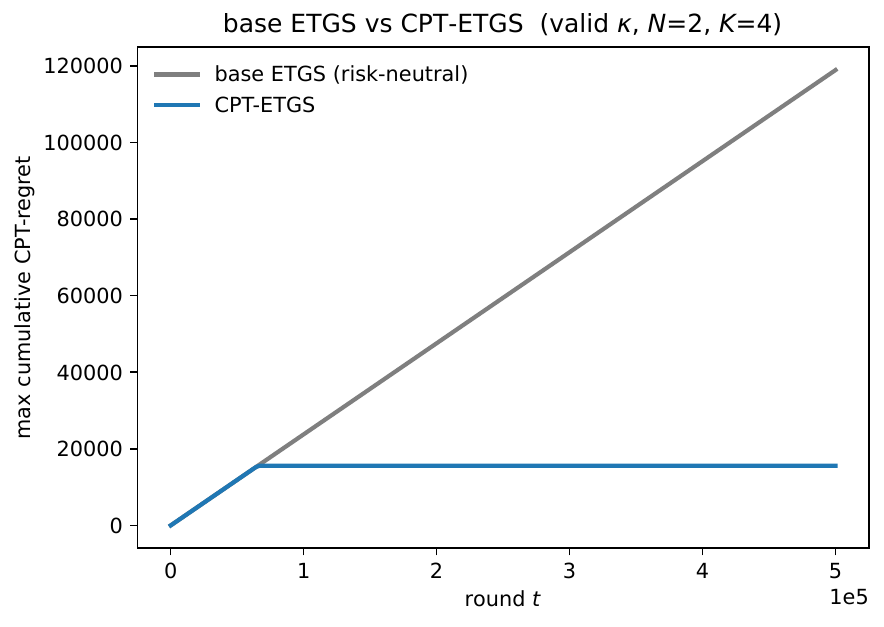}
    \caption{Rank-swap market ($N{=}2$, $K{=}4$). Base \texttt{ETGS} never
    settles $\Rightarrow$ \emph{linear} CPT-regret; \texttt{CPT-ETGS} saturates
    (Thm~\ref{th:cpt_etgs}).}
    \label{fig:base}
  \end{subfigure}\hfill
  \begin{subfigure}[t]{0.32\linewidth}
    \centering
    \includegraphics[width=\linewidth]{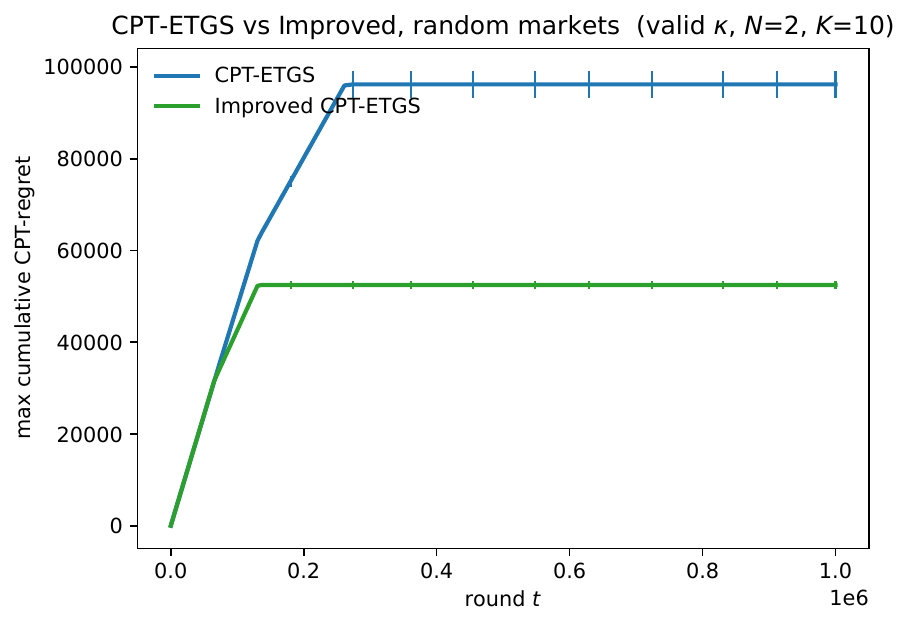}
    \caption{Many-arm random market ($N{=}2$, $K{=}10$; good-arm CPT values drawn
    at random each run). Improved \texttt{CPT-ETGS} eliminates fillers early and
    commits at a lower regret (Thm~\ref{th:improved}).}
    \label{fig:improved-rand}
  \end{subfigure}\hfill
  \begin{subfigure}[t]{0.32\linewidth}
    \centering
    \includegraphics[width=\linewidth]{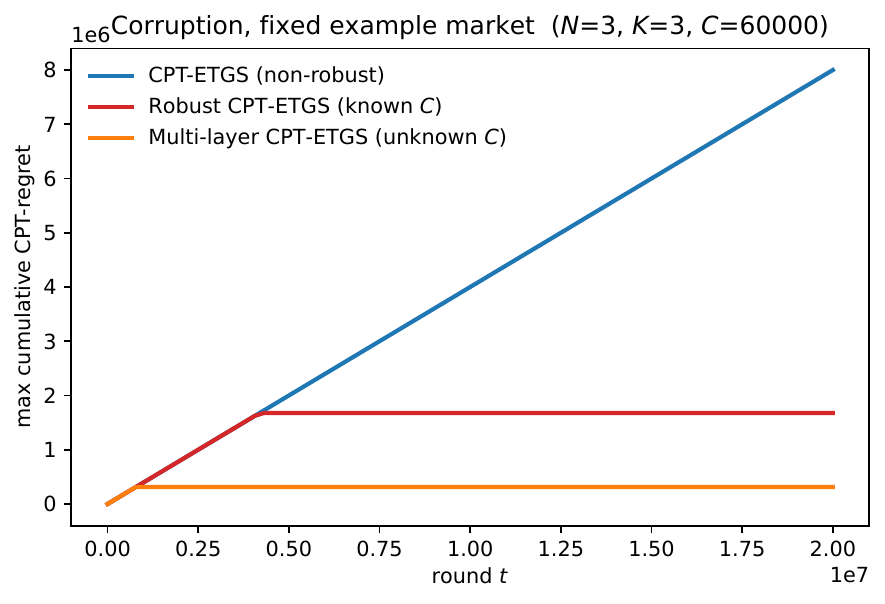}
    \caption{Adversarial corruption ($N{=}K{=}3$, $C{=}6\times10^{4}$; fixed
    example market, CPT values rotating $\{0.9,0.5,0.1\}$). Non-robust is linear;
    robust (known $C$) and Multi-layer (unknown $C$) saturate
    (Thms~\ref{th:known}--\ref{th:unknown}).}
    \label{fig:example}
  \end{subfigure}
  \caption{\textbf{Player-optimal stable CPT-regret} (max over players, averaged
  over $100$ runs; error bars are $\pm$ one standard error)  (\subref{fig:base})~Ignoring the CPT distortion is
  harmful. (\subref{fig:improved-rand})~Adaptive elimination lowers
  the regret. (\subref{fig:example})~Under corruption,
  both robust variants stay sub-linear while the
  non-robust regret diverges. }
  \label{fig:exp}
\end{figure}

Rewards are Bernoulli that is arm $a_j$ yields $v_{i,j}$ with probability $q_{i,j}$ and
$0$ otherwise. Thus
$\mu_{i,j}=w(q_{i,j})v_{i,j}$ and
$\bar\mu_{i,j}=q_{i,j}v_{i,j}$.
We report the maximum player-optimal stable CPT-regret over
players, averaged over $100$ runs with $\pm1$ standard error bars.

\paragraph{Ignoring CPT is harmful (Fig.~\ref{fig:base}).}
We construct a rank-swap market ($N=2$, $K=4$) where CPT and risk-neutral
preferences induce different player-optimal stable matchings. Base
\texttt{ETGS}, which estimates risk-neutral means, fails to identify the
CPT-optimal matching and incurs linear CPT-regret
. In contrast, \texttt{CPT-ETGS} commits to the correct
matching and its regret saturates
, consistent with
Theorem~\ref{th:cpt_etgs}.

\paragraph{Adaptive elimination helps (Fig.~\ref{fig:improved-rand}).}
For $N=2$, $K=10$, each player has two good arms and eight filler arms with $\mu_{i,j} \approx0.005$. In every run, good-arm $\mu_{i,j}$ values are redrawn
(top value in $[0.45,0.65]$, second-best $0.15$--$0.30$ lower), and arm rankings
are sampled independently. Both algorithms eventually commit, but Improved
\texttt{CPT-ETGS} quickly eliminates fillers, commits earlier, and achieves a
lower regret 
matching Theorem~\ref{th:improved}.

\paragraph{Adversarial corruption (Fig.~\ref{fig:example}).}
We consider a fixed $N=K=3$ market whose $\mu_{i,j}$ are cyclic permutations of
$\{0.9,0.5,0.1\}$. An adversary with
per-player budget $C=6\times10^4$ suppresses rewards on the optimal arms. The
non-robust \texttt{CPT-ETGS} is misled into a wrong matching and suffers linear
regret . The robust algorithm with known corruption
budget commits correctly and saturates
, while the multi-layer algorithm for unknown $C$
also remains sub-linear and attains an even lower regret
, in agreement with
Theorems~\ref{th:known}--\ref{th:unknown}.

\section{Conclusion}
We consider human centric risk sensitive measure, CPT for matching markets and analyze learning algorithms. Exploring other human aligned setups in a fair matching market is kept as future endeavor.

\paragraph{Use of Generative AI:}
AI tools were used only to improve the clarity and presentation of the paper, including shortening, rephrasing, and editing text. They were not used for research ideas, theoretical developments, or conclusions.


\printbibliography

@inproceedings{kong2023player,
  title={Player-optimal stable regret for bandit learning in matching markets},
  author={Kong, Fang and Li, Shuai},
  booktitle={Proceedings of the 2023 Annual ACM-SIAM Symposium on Discrete Algorithms (SODA)},
  pages={1512--1522},
  year={2023},
  organization={SIAM}
}

@article{kolla2016bandit,
  title={Bandit algorithms to emulate human decision making using probabilistic distortions},
  author={Kolla, Ravi Kumar and Gopalan, Aditya and Jagannathan, Krishna and Fu, Michael and Marcus, Steve and others},
  journal={arXiv preprint arXiv:1611.10283},
  year={2016}
}

@inproceedings{
wubandit,
title={Bandit Learning in Matching Markets Robust to Adversarial Corruptions},
author={Zheshun Wu and Jinhang Zuo and Zenglin Xu and Fang Kong},
booktitle={The Fourteenth International Conference on Learning Representations},
year={2026},
url={https://openreview.net/forum?id=CoWJc5ofDO}
}

@article{kong2024improved,
  title={Improved analysis for bandit learning in matching markets},
  author={Kong, Fang and Wang, Zilong and Li, Shuai},
  journal={Advances in Neural Information Processing Systems},
  volume={37},
  pages={91904--91929},
  year={2024}
}

@InProceedings{pmlr-v48-la16,
  title = 	 {Cumulative Prospect Theory Meets Reinforcement Learning: Prediction and Control},
  author = 	 {L.A., Prashanth and Jie, Cheng and Fu, Michael and Marcus, Steve and Szepesvari, Csaba},
  booktitle = 	 {Proceedings of The 33rd International Conference on Machine Learning},
  pages = 	 {1406--1415},
  year = 	 {2016},
  editor = 	 {Balcan, Maria Florina and Weinberger, Kilian Q.},
  volume = 	 {48},
  series = 	 {Proceedings of Machine Learning Research},
  address = 	 {New York, New York, USA},
  month = 	 {20--22 Jun},
  publisher =    {PMLR},
  url = 	 {https://proceedings.mlr.press/v48/la16.html}
}

@article{liu2021bandit,
  title={Bandit learning in decentralized matching markets},
  author={Liu, Lydia T and Ruan, Feng and Mania, Horia and Jordan, Michael I},
  journal={Journal of Machine Learning Research},
  volume={22},
  number={211},
  pages={1--34},
  year={2021}
}

@inproceedings{liu2020competing,
  title={Competing bandits in matching markets},
  author={Liu, Lydia T and Mania, Horia and Jordan, Michael},
  booktitle={International Conference on Artificial Intelligence and Statistics},
  pages={1618--1628},
  year={2020},
  organization={PMLR}
}

@InProceedings{pmlr-v130-sankararaman21a,
  title = 	 { Dominate or Delete: Decentralized Competing Bandits in Serial Dictatorship },
  author =       {Sankararaman, Abishek and Basu, Soumya and Abinav Sankararaman, Karthik},
  booktitle = 	 {Proceedings of The 24th International Conference on Artificial Intelligence and Statistics},
  pages = 	 {1252--1260},
  year = 	 {2021},
  editor = 	 {Banerjee, Arindam and Fukumizu, Kenji},
  volume = 	 {130},
  series = 	 {Proceedings of Machine Learning Research},
  month = 	 {13--15 Apr},
  publisher =    {PMLR},
  url = 	 {https://proceedings.mlr.press/v130/sankararaman21a.html}
}

@inproceedings{basu2021beyond,
  title={Beyond $\log^2(T)$ regret for decentralized bandits in matching markets},
  author={Basu, Soumya and Sankararaman, Karthik Abinav and Sankararaman, Abishek},
  booktitle={International Conference on Machine Learning},
  pages={705--715},
  year={2021},
  organization={PMLR}
}

@inproceedings{pagare2024explore,
  title={Explore-then-commit algorithms for decentralized two-sided matching markets},
  author={Pagare, Tejas and Ghosh, Avishek},
  booktitle={2024 IEEE International Symposium on Information Theory (ISIT)},
  pages={2092--2097},
  year={2024},
  organization={IEEE}
}

@ARTICLE{ghosh_nonstationary,
  author={Ghosh, Avishek and Sankararaman, Abishek and Ramchandran, Kannan and Javidi, Tara and Mazumdar, Arya},
  journal={IEEE Transactions on Information Theory}, 
  title={Competing Bandits in Non-Stationary Matching Markets}, 
  year={2024},
  volume={70},
  number={4},
  pages={2831-2850},
  doi={10.1109/TIT.2024.3352228}
  }

@article{meena_market,
author = {Jagadeesan, Meena and Wei, Alexander and Wang, Yixin and Jordan, Michael I. and Steinhardt, Jacob},
title = {Learning Equilibria in Matching Markets with Bandit Feedback},
year = {2023},
issue_date = {June 2023},
publisher = {Association for Computing Machinery},
address = {New York, NY, USA},
volume = {70},
number = {3},
issn = {0004-5411},
url = {https://doi.org/10.1145/3583681},
doi = {10.1145/3583681},
journal = {J. ACM},
month = may,
articleno = {19},
numpages = {46}
}

@article{Rockafellar2000OptimizationOC,
  title={Optimization of conditional value-at risk},
  author={R. Tyrrell Rockafellar and Stanislav Uryasev},
  journal={Journal of Risk},
  year={2000},
  volume={3},
  pages={21-41},
  url={https://api.semanticscholar.org/CorpusID:854622}
}

@book{Kahneman2011-KAHTFA-2,
	address = {New York: New York},
	author = {Daniel Kahneman},
	editor = {},
	publisher = {Farrar, Straus and Giroux},
	title = {Thinking, Fast and Slow},
	year = {2011}
}

@article{e5a1bb8f-41b7-35c6-95cd-8b366d3e99bc,
 ISSN = {00221082, 15406261},
 URL = {http://www.jstor.org/stable/2975974},
 author = {Harry Markowitz},
 journal = {The Journal of Finance},
 number = {1},
 pages = {77--91},
 publisher = {[American Finance Association, Wiley]},
 title = {Portfolio Selection},
 urldate = {2026-01-25},
 volume = {7},
 year = {1952}
}

@article{https://doi.org/10.1111/1467-9965.00068,
author = {Artzner, Philippe and Delbaen, Freddy and Eber, Jean-Marc and Heath, David},
title = {Coherent Measures of Risk},
journal = {Mathematical Finance},
volume = {9},
number = {3},
pages = {203-228},
doi = {https://doi.org/10.1111/1467-9965.00068},
url = {https://onlinelibrary.wiley.com/doi/abs/10.1111/1467-9965.00068},
eprint = {https://onlinelibrary.wiley.com/doi/pdf/10.1111/1467-9965.00068},
year = {1999}
}

@article{ACERBI20021505,
title = {Spectral measures of risk: A coherent representation of subjective risk aversion},
journal = {Journal of Banking \& Finance},
volume = {26},
number = {7},
pages = {1505-1518},
year = {2002},
issn = {0378-4266},
doi = {https://doi.org/10.1016/S0378-4266(02)00281-9},
url = {https://www.sciencedirect.com/science/article/pii/S0378426602002819},
author = {Carlo Acerbi}
}

@misc{tan2022surveyriskawaremultiarmedbandits,
      title={A Survey of Risk-Aware Multi-Armed Bandits}, 
      author={Vincent Y. F. Tan and Prashanth L. A. and Krishna Jagannathan},
      year={2022},
      eprint={2205.05843},
      archivePrefix={arXiv},
      primaryClass={stat.ML},
      url={https://arxiv.org/abs/2205.05843}, 
}

@article{Hu2018UtilitybasedSR,
  title={Utility‐based shortfall risk: Efficient computations via Monte Carlo},
  author={Zhaolin Hu and Dali Zhang},
  journal={Naval Research Logistics (NRL)},
  year={2018},
  volume={65},
  pages={378 - 392},
  url={https://api.semanticscholar.org/CorpusID:126005263}
}

@article{tversky1992advances,
  title={Advances in prospect theory: Cumulative representation of uncertainty},
  author={Tversky, Amos and Kahneman, Daniel},
  journal={Journal of Risk and uncertainty},
  volume={5},
  number={4},
  pages={297--323},
  year={1992},
  publisher={Springer}
}

@article{DBLP:journals/corr/GopalanAFM16,
  author       = {Aditya Gopalan and
                  Prashanth L. A. and
                  Michael C. Fu and
                  Steven I. Marcus},
  title        = {Weighted bandits or: How bandits learn distorted values that are not
                  expected},
  journal      = {CoRR},
  volume       = {abs/1611.10283},
  year         = {2016},
  url          = {http://arxiv.org/abs/1611.10283},
  eprinttype    = {arXiv},
  eprint       = {1611.10283},
  bibsource    = {dblp computer science bibliography, https://dblp.org}
}

@InProceedings{pmlr-v258-tatli25a,
  title = 	 {Risk-sensitive Bandits: Arm Mixture Optimality and Regret-efficient Algorithms},
  author =       {Tatl{\i}, Meltem and Mukherjee, Arpan and A., Prashanth L. and Shanmugam, Karthikeyan and Tajer, Ali},
  booktitle = 	 {Proceedings of The 28th International Conference on Artificial Intelligence and Statistics},
  pages = 	 {3871--3879},
  year = 	 {2025},
  editor = 	 {Li, Yingzhen and Mandt, Stephan and Agrawal, Shipra and Khan, Emtiyaz},
  volume = 	 {258},
  series = 	 {Proceedings of Machine Learning Research},
  month = 	 {03--05 May},
  publisher =    {PMLR},
  url = 	 {https://proceedings.mlr.press/v258/tatli25a.html}
}

@inproceedings{zhang_click,
author = {Zhang, Linfeng and Guan, Yong},
title = {Detecting Click Fraud in Pay-Per-Click Streams of Online Advertising Networks},
year = {2008},
isbn = {9780769531724},
publisher = {IEEE Computer Society},
address = {USA},
url = {https://doi.org/10.1109/ICDCS.2008.98},
doi = {10.1109/ICDCS.2008.98},
booktitle = {Proceedings of the 2008 The 28th International Conference on Distributed Computing Systems},
pages = {77–84},
numpages = {8},
series = {ICDCS '08}
}

@article{oentaryo_click,
author = {Oentaryo, Richard and Lim, Ee-Peng and Finegold, Michael and Lo, David and Zhu, Feida and Phua, Clifton and Cheu, Eng-Yeow and Yap, Ghim-Eng and Sim, Kelvin and Nguyen, Minh Nhut and Perera, Kasun and Neupane, Bijay and Faisal, Mustafa and Aung, Zeyar and Woon, Wei Lee and Chen, Wei and Patel, Dhaval and Berrar, Daniel},
title = {Detecting click fraud in online advertising: a data mining approach},
year = {2014},
issue_date = {January 2014},
publisher = {JMLR.org},
volume = {15},
number = {1},
issn = {1532-4435},
journal = {J. Mach. Learn. Res.},
month = jan,
pages = {99–140},
numpages = {42}
}

@article{arora2012online,
  title={Online bandit learning against an adaptive adversary: from regret to policy regret},
  author={Arora, Raman and Dekel, Ofer and Tewari, Ambuj},
  journal={arXiv preprint arXiv:1206.6400},
  year={2012}
}

@inproceedings{yang_adversarial,
 author = {yang, lin and Hajiesmaili, Mohammad and Talebi, Mohammad Sadegh and Lui, John C. S. and Wong, Wing Shing},
 booktitle = {Advances in Neural Information Processing Systems},
 editor = {H. Larochelle and M. Ranzato and R. Hadsell and M.F. Balcan and H. Lin},
 pages = {19943--19952},
 publisher = {Curran Associates, Inc.},
 title = {Adversarial Bandits with Corruptions: Regret Lower Bound and No-regret Algorithm},
 volume = {33},
 year = {2020}
}

@inproceedings{niss2020you,
  title={What You See May Not Be What You Get: UCB Bandit Algorithms Robust to epsilon Contamination},
  author={Niss, Laura and Tewari, Ambuj},
  booktitle={Conference on Uncertainty in Artificial Intelligence},
  pages={450--459},
  year={2020},
  organization={PMLR}
}

@InProceedings{fitness_sankararaman22a,
  title = 	 {{FITNESS}: ({F}ine Tune on New and Similar Samples) to detect anomalies in streams with drift and outliers},
  author =       {Sankararaman, Abishek and Narayanaswamy, Balakrishnan and Singh, Vikramank Y and Song, Zhao},
  booktitle = 	 {Proceedings of the 39th International Conference on Machine Learning},
  pages = 	 {19153--19177},
  year = 	 {2022},
  editor = 	 {Chaudhuri, Kamalika and Jegelka, Stefanie and Song, Le and Szepesvari, Csaba and Niu, Gang and Sabato, Sivan},
  volume = 	 {162},
  series = 	 {Proceedings of Machine Learning Research},
  month = 	 {17--23 Jul},
  publisher =    {PMLR},
  url = 	 {https://proceedings.mlr.press/v162/sankararaman22a.html}
}

@inproceedings{robust_abishek,
author = {Sankararaman, Abishek and Narayanaswamy, Balakrishnan (Murali)},
title = {Online robust non-stationary estimation},
year = {2023},
publisher = {Curran Associates Inc.},
address = {Red Hook, NY, USA},
booktitle = {Proceedings of the 37th International Conference on Neural Information Processing Systems},
articleno = {2197},
numpages = {39},
location = {New Orleans, LA, USA},
series = {NIPS '23}
}

@article{gale1962college,
  title={College admissions and the stability of marriage},
  author={Gale, David and Shapley, Lloyd S},
  journal={The American mathematical monthly},
  volume={69},
  number={1},
  pages={9--15},
  year={1962},
  publisher={Taylor \& Francis}
}

@article{LAI19854,
title = {Asymptotically efficient adaptive allocation rules},
journal = {Advances in Applied Mathematics},
volume = {6},
number = {1},
pages = {4-22},
year = {1985},
issn = {0196-8858},
doi = {https://doi.org/10.1016/0196-8858(85)90002-8},
url = {https://www.sciencedirect.com/science/article/pii/0196885885900028},
author = {T.L Lai and Herbert Robbins}
}

\appendix
\clearpage
\section{Additional Related Work}
\emph{Markets and Bandits:} A rich line of work studies bandit learning in two-sided matching markets in recent years~\cite{liu2020competing,liu2021bandit,kong2023player,kong2024improved,basu2021beyond,pmlr-v130-sankararaman21a,ghosh_nonstationary,meena_market,pagare2024explore} using various assumptions like serial dictatorship, $\alpha$-condition etc. All of these works adopt expected reward as the performance criterion and none incorporate behavioral risk measures such as cumulative prospect theory (CPT).

\emph{Risk Sensitive ML:} Risk-sensitive objectives in bandits have been studied through mean-
variance~\cite{e5a1bb8f-41b7-35c6-95cd-8b366d3e99bc}, 
spectral risk measures~\cite{ACERBI20021505}, and utility-based shortfall risk~\cite{Hu2018UtilitybasedSR}. CPT-based regret for single-player bandits was introduced in~\cite{DBLP:journals/corr/GopalanAFM16,kolla2016bandit}, while~\cite{pmlr-v258-tatli25a}
studied arm mixture optimality under risk-sensitive objectives.

\emph{Adversarial Learning:} On the adversarial side, corrupted bandits have been widely studied in
single-player settings~\cite{arora2012online,yang_adversarial}, and~\cite{wubandit} recently extended corruption
robustness to matching markets. Our work bridges these two directions by jointly addressing CPT-based risk preferences and adversarial
corruption in competitive matching markets.

\section{Details of Algorithmic procedures}
\label{app:phase_1_3}

In this section, we provide the details of first and third phase mentioned in Algorithm \ref{alg:cpt_etgs}.
\subsection{Phase~1}
In the first phase, each player obtains a unique index that determines the
order in which players explore arms during the exploration phase. All players
initially propose to the first arm $a_1$. Since each arm accepts only the
player it prefers most, the accepted player receives index $1$. In subsequent
rounds, the players who already were assigned an index propose to arm $a_2$, and the remaining players continue proposing to $a_1$,  the accepted
player for arm $a_1$ receives the next available index. After $N$ rounds, every player
obtains a distinct index.
\subsection{Phase~3}
In the final phase, players execute a decentralized Gale-Shapley procedure
using their estimated preference rankings obtained in Phase~2. Each player sequentially proposes
to arms according to its estimated ranking until no rejection occurs. When
the rankings of the top $N$ arms are estimated correctly, this process
converges to the player-optimal stable matching.

\section{Proof of Theorem \ref{th:cpt_etgs}}
\label{app:proof_cpt_etgs}
\begin{proof}
Let $\hat{\mu}_{i,j}(t)$ denote the estimated preference value and $T_{i,j}(t)$ the number of observations of arm $a_j$ for player $p_i$ at the end of round $t$, and  $\ell_{\max}$ the largest sub-phase index in Phase~2.

Define the bad event $\mathcal{F}$ as follows:
\[
\mathcal{F}
=
\Big\{
\exists\, t \in [T],\;
i \in [N],\;
j \in [K]
:
\big|
\hat{\mu}_{i,j}(t) - \mu_{i,j}
\big|
>
LM\!\left(
\frac{3\log T}{2T_{i,j}(t)}
\right)^{\alpha/2}
\Big\}.
\]
The player optimal stable CPT-regret of $p_i$ under Algorithm~\ref{alg:cpt_etgs} satisfies
\begin{align*}
\text{CPT-}\mathrm{Reg}_i(T)
&=
\sum_{j \in [K],j\neq m_i^*}
\mathbb{E}\!\left[T_{i,j}(T)\right]\Delta_{i,m_i^*,j}
\nonumber\\
&\le
N\Delta_{i,\max}
+
\mathbb{E}\!\left[
\sum_{t=N+1}^{T}
\mathbf{1}\{\bar{A}_i(t)\neq m_i^*\}
\,\middle|\, \neg\mathcal{F}
\right]\Delta_{i,\max}
+
\mathbb{P}(\mathcal{F})\,T\,\Delta_{i,\max}.
\\
&\le
N\Delta_{i,\max}
+
\mathbb{E}\!\left[
\sum_{\ell=1}^{\ell_{\max}} (2^\ell+1)
+ N^2
\,\middle|\, \neg\mathcal{F}
\right]\Delta_{i,\max}
+
2NK\,\Delta_{i,\max}.
\\
&\le
N\Delta_{i,\max}
+
\Bigg(
\frac{3K\log T}{\Delta^{2/\alpha}}(4LM)^{2/\alpha}
+
\log\!\left(
\frac{3K\log T}{\Delta^{2/\alpha}}(4LM)^{2/\alpha}
\right)
\Bigg)\Delta_{i,\max}
\nonumber\\
&\quad
+ N^2\Delta_{i,\max}
+ 2NK\Delta_{i,\max}.
\end{align*}
The last inequality follows from the bound on $\ell_{\max}$ in Lemma~\ref{lemma:cpt-phase2}. 
The first, second, and third terms correspond to the regret incurred in the first, second, and third phases respectively, while the fourth term captures the regret resulting from the occurrence of the bad event $\mathcal{F}$.
\end{proof}
\begin{lemma}[{\cite[Theorem~1]{kolla2016bandit}}]
\label{lemma:prob_bound}
(Sample complexity of estimating weight-distorted reward).
Assume Assumption~\ref{ass:bounded} and Assumption~\ref{ass:holder}. Then for any $\varepsilon > 0$ and any
$j \in \{1,\dots,K\}$, we have
$\mathbb{P}\!\left(
\left| \hat{\mu}_{i,j,m} - \mu_{i,j} \right| > \varepsilon
\right)
\le
2 \exp\!\left(
-2m \left(\frac{\varepsilon}{LM}\right)^{2/\alpha}
\right)$, where m is the number of samples of arm j observed by player i.
\end{lemma}
\begin{lemma} 
\begin{align*}
    \mathbb{P}(\mathcal{F}) \le \frac{2NK}{T}
\end{align*}

\end{lemma}

\begin{proof}
Fix $t \in [T]$, $i \in [N]$, and $j \in [K]$.
Applying Lemma~\ref{lemma:prob_bound} with
$m = T_{i,j}(t)$ and
$
\varepsilon
=
LM\left(
\frac{3\log T}{2T_{i,j}(t)}
\right)^{\alpha/2},
$
we obtain
\[
\mathbb{P}\!\left(
\left|
\hat{\mu}_{i,j}(t) - \mu_{i,j}
\right|
>
LM\left(
\frac{3\log T}{2T_{i,j}(t)}
\right)^{\alpha/2}
\right)
\le
2 \exp\!\left(
-2T_{i,j}(t)
\left(
\frac{
LM\left(\frac{3\log T}{2T_{i,j}(t)}\right)^{\alpha/2}
}{LM}
\right)^{2/\alpha}
\right).
\]

\[
\mathbb{P}\!\left(
\left|
\hat{\mu}_{i,j}(t) - \mu_{i,j}
\right|
>
LM\left(
\frac{3\log T}{2T_{i,j}(t)}
\right)^{\alpha/2}
\right)
\le
2 e^{-3\log T}
=
\frac{2}{T^3}.
\]

 \begin{align*}
  \mathbb{P}(\mathcal{F})
  &\le
  \sum_{t=1}^{T}
  \sum_{i=1}^{N}
  \sum_{j=1}^{K}
  \sum_{s=1}^{t}
  \mathbb{P}\!\left(
  T_{i,j}(t)=s,\,
  \left|\hat{\mu}_{i,j}(t)-\mu_{i,j}\right|
  >
  LM\left(\frac{3\log T}{2s}\right)^{\alpha/2}
  \right)
  \\
  &\le
  \sum_{t=1}^{T}
  \sum_{i=1}^{N}
  \sum_{j=1}^{K}
  t \cdot \frac{2}{T^3}
  =
  \frac{2NK}{T^3}
  \cdot \frac{T(T+1)}{2}
  \le
  \frac{2NK}{T}.
  \end{align*}
\end{proof}
\begin{lemma}[{\cite[Lemma3]{kong2023player}}]
\label{lemma:offline_gs}
In the offline GS algorithm, at most $N$ arms are proposed to by players before the algorithm stops. 
Thus, the player-optimal stable arm of each player must be among its first $N$-ranked arms. 
Moreover, the GS algorithm reaches player-optimal stability in at most $N^2$ steps.
\end{lemma}
\begin{lemma}
\label{lemma:cpt-phase2}
Conditional on $\neg\mathcal{F}$, Phase~2 proceeds in at most $\ell_{\max}$ sub-phases where
\begin{equation}
\label{eq:phase2-len}
\ell_{\max}
=
\min
\left\{
\ell :
\sum_{r=1}^{\ell} 2^r
\ge
K \cdot
\frac{3\log T}{2}
\left(\frac{4LM}{\Delta}\right)^{2/\alpha}
\right\}.
\end{equation}
This implies
$
\sum_{r=1}^{\ell_{\max}} 2^r
\le
3K\log T
\left(\frac{4LM}{\Delta}\right)^{2/\alpha},
$
and
$
\ell_{\max}
=
\log\!\left(
3K\log T
\left(\frac{4LM}{\Delta}\right)^{2/\alpha}
\right),
$
since the sub-phase length grows exponentially.
All players enter Phase~3 simultaneously at the end of sub-phase $\ell_{\max}$.
\end{lemma}
\begin{proof}
We know that in the exploration part of Phase~2 all players propose to arms in a round robin order due to which there are no collisions. Consequently, by the end of sub-phase $\ell_{\max}$ defined in~\eqref{eq:phase2-len}, the number of samples satisfies
\[
T_{i,j}
\ge
\frac{3\log T}{2}
\left(\frac{4LM}{\Delta}\right)^{2/\alpha}
\qquad
\forall i \in [N],\; j \in [K].
\]

According to Lemma~\ref{lemma:ti_cpt_etgs}, once this sampling threshold is reached, player $p_i$ is able to determine a permutation $\sigma_i$ such that
\[
W\text{-LCB}_{i,\sigma_{i,j}}(t)
>
W\text{-UCB}_{i,\sigma_{i,j+1}}(t)
\quad \forall j \in [N],
\]
and
\[
W\text{-LCB}_{i,\sigma_{i,N}}(t)
>
W\text{-UCB}_{i,\sigma_{i,j}}(t)
\quad \forall j = N+1,\dots,K.
\]

During the monitoring round of sub-phase $\ell_{\max}$, each player proposes to the arm corresponding to its assigned index and observes that $|O_{\ell_{\max}}| = N$. 
This confirms that all players are accepted, after which they proceed simultaneously to Phase~3.
\end{proof}
\begin{lemma}
Conditional on $\neg\mathcal{F}$,
$W\text{-UCB}_{i,j}(t) < W\text{-LCB}_{i,j'}(t)$
 implies  
$\mu_{i,j} < \mu_{i,j'} .$
\end{lemma}
\begin{proof}
Following the definition of $W$-LCB and $W$-UCB in eq~\ref{eq:wucb} and the condition that $\neg\mathcal{F}$ holds, we have
\begin{align*}                                                                                                                        
  W\text{-LCB}_{i,j}(t)
  &=                                       
  \hat{\mu}_{i,j}(t)
  -
  LM\!\left(
  \frac{3\log T}{2T_{i,j}(t)}
  \right)^{\alpha/2}
  \le
  \mu_{i,j}, \\
  \mu_{i,j}
  &\le
  \hat{\mu}_{i,j}(t)
  +
  LM\!\left(
  \frac{3\log T}{2T_{i,j}(t)}
  \right)^{\alpha/2}
  =
  W\text{-UCB}_{i,j}(t).
  \end{align*}

then
\[
\mu_{i,j}
\le
W\text{-UCB}_{i,j}(t)
<
W\text{-LCB}_{i,j'}(t)
\le
\mu_{i,j'},
\]
which completes the proof.
\end{proof}
\begin{lemma}
\label{lemma:ti_cpt_etgs}
In round $t$, let
$
T_i(t)=\min_{j\in[K]} T_{i,j}(t),
\bar T_i
=
\frac{3\log T}{2}
\left(\frac{4LM}{\Delta}\right)^{2/\alpha}.
$
Conditional on $\neg\mathcal{F}$, if $T_i(t)>\bar T_i$, then
$
W\text{-LCB}_{i,\rho_{i,j}}(t)
>
W\text{-UCB}_{i,\rho_{i,j+1}}(t)
\quad \forall j\in[N],
$
and
$
W\text{-LCB}_{i,\rho_{i,N}}(t)
>
W\text{-UCB}_{i,\rho_{i,j}}(t)
\quad \forall j=N+1,\dots,K.
$
\end{lemma}

\begin{proof}
We prove the claim by contradiction. Suppose that there exists $k \in [N]$ such that
\[
W\text{-LCB}_{i,\rho_{i,k}}(t)
\le
W\text{-UCB}_{i,\rho_{i,k+1}}(t),
\]
or there exists $k = N+1,\dots,K$ for which
\[
W\text{-LCB}_{i,\rho_{i,N}}(t)
\le
W\text{-UCB}_{i,\rho_{i,k}}(t).
\]
Let $j$ denote the arm appearing on the right-hand side of the inequality and $j'$ the arm on the left-hand side.

Conditioned on the event $\neg\mathcal{F}$ and using the definitions of the confidence bounds in equation~\ref{eq:wucb}, we have
\[
\mu_{i,j'}
-
2LM\!\left(
\frac{3\log T}{2T_i(t)}
\right)^{\alpha/2}
\le
W\text{-LCB}_{i,j'}(t)
\le
W\text{-UCB}_{i,j}(t)
\le
\mu_{i,j}
+
2LM\!\left(
\frac{3\log T}{2T_i(t)}
\right)^{\alpha/2}.
\]

This implies
\[
\Delta_{i,j,j'}
\le
4LM \left( \frac{3\log T}{2T_i(t)} \right)^{\alpha/2}.
\]

Since $\Delta_{i,j,j'} \ge \Delta$, it follows that
\[
T_i(t)
\le
\frac{3\log T}{2}
\left(\frac{4LM}{\Delta}\right)^{2/\alpha}
=
\bar T_i,
\]
which contradicts the assumption that $T_i(t) > \bar T_i$. The proof is therefore complete.
\end{proof}

\section{Proof of Theorem \ref{th:improved}}
\label{app:improved}
We begin by introducing the concentration failure event. Let
\[
\mathcal{F}=\left\{\exists\,1\le t\le T,\, i\in[N],\, j\in[K]:
\left|\hat{\mu}_{i,j}(t)-\mu_{i,j}\right|>
LM\left(\frac{3\log T}{2T_{i,j}(t)}\right)^{\alpha/2}\right\}
\]
denote the event in which the empirical reward estimate deviates from the true mean beyond the prescribed confidence width, for at least one time step and one player-arm pair. We now proceed to establish an upper bound on the cumulative regret.

\begin{align*}
\text{CPT-}\mathrm{Reg}_i(T)
&= \mathbb{E}\!\left[\sum_{t=1}^{T}(\mu_{i,m_i^*}-\mu_{i,\bar{A}_i(t)})\right] \\
&\le
  \mathbb{E}\!\left[\sum_{t=1}^{T}(\mu_{i,m_i^*}-\mu_{i,\bar{A}_i(t)}){\mid\neg\mathcal{F}}\right]
+ \mathbb{P}(\mathcal{F})\cdot T\cdot \mu_{i,m_i^*} \\
&\le \mathbb{E}\!\left[
\sum_{t=1}^{T}\sum_{a_j}
\mathbf{1}\{\bar{A}_i(t)=a_j\}\cdot \Delta_{i,m_i^*,j}
{\mid\neg\mathcal{F}}
\right]   \\
&\quad 
+\mathbb{E}\!\left[
\sum_{t=1}^{T}\mathbf{1}\{\bar{A}_i(t)=\emptyset\}\cdot \mu_{i,m_i^*}{\mid\neg\mathcal{F}}
\right]
+ \mathbb{P}(\mathcal{F})\cdot T\cdot \mu_{i,m_i^*} \\
&\le \frac{3N^2\log T}{2}
\left(\frac{4LM}{\Delta}\right)^{2/\alpha}\cdot \Delta_{i,max} + \sum_{a_j\in\mathcal{K}\setminus\{a_{m_i^*}\}} 
\frac{3\log T}{2}
 \left(\frac{(4LM)^{2/\alpha}}{\Delta_{i,m_i^*,j}^{{2/\alpha}-1}}\right) 
 \\
&\quad 
+{3N^2\log T}
\left(\frac{4LM}{\Delta}\right)^{2/\alpha}\cdot \Delta_{i,max} +2NK\cdot \Delta_{i,max}\\
&= O\!\left({N^2\log T}
\left(\frac{LM}{\Delta}\right)^{2/\alpha}\cdot \Delta_{i,max}+ {K\log T}
\left(\frac{(LM)^{2/\alpha}}{\Delta^{{2/\alpha}-1}}\right)\right)
\end{align*}

where the final inequality follows by combining Lemma~\ref{lem:subopt_arm_regret}, Lemma~\ref{lem:collision_regret}, and Lemma~\ref{lem:failure_prob}.

\begin{lemma}
\label{lem:subopt_arm_regret}
The expected regret due to playing sub-optimal arms, conditioned on the good event, satisfies
\begin{align*}
\mathbb{E}\!\left[
\sum_{t=1}^{T}\sum_{a_j \in [K]}
\mathbf{1}\{\bar{A}_i(t)=a_j\}\cdot \Delta_{i,m_i^*,j}
\,\middle|\,\neg\mathcal{F}
\right]
&\le \frac{3N^2\log T}{2}
\left(\frac{4LM}{\Delta}\right)^{2/\alpha}\cdot \Delta_{i,\max} \\
&\quad +  \sum_{a_j\in\mathcal{K}\setminus\{a_{m_i^*}\}} 
\frac{3\log T}{2}
\left(\frac{(4LM)^{2/\alpha}}{\Delta_{i,m_i^*,j}^{{2/\alpha}-1}}\right).
\end{align*}
\end{lemma}
\begin{proof}
Observe that player $p_i$ revises its available set $\mathcal{A}_i$
each time another player $p_{i'}$ switches $E_{i'}=\text{False}$ while
$\pi_{A_{i'},i'}>\pi_{A_i,i'}$, as prescribed by Lines~12--13 of Algorithm~\ref{alg:improved}. We write $t_s$ for the round
at which this revision occurs for the $s$-th time, and set $t_0=1$ by convention.

Intuitively, whenever a player $p_{i'}$
transitions $E_{i'}$ to False, this indicates that $p_{i'}$ has successfully identified its most preferred
arm among the currently available options. In conjunction with the good event $\neg\mathcal{F}$ and Lemma~\ref{lem:ucb_ordering},
the arm committed to by each player during exploration is guaranteed to be
their true top choice among remaining arms. This observation reveals that the Improved \texttt{CPT-ETGS} algorithm operates as an online realization
of the Gale-Shapley procedure, the $s$-th transition of $E_{i'}$ to False by player $p_{i'}$ mirrors
the event where $p_{i'}$ proposes to its $s$-th most preferred arm in the
classical offline GS. It then follows from Lemma~\ref{lem:gs_proposals} that the total number of proposals across
all players before stability is reached cannot exceed $N-1$. Hence, the
revision at Lines~12--13 is triggered no more than $N-1$ times for player $p_i$.

Within each interval indexed by $s$, spanning rounds $t_s$ through $t_{s+1}$, player $p_i$
iterates over all available arms using a round-robin scheme, progressively discarding
inferior arms until $N$ candidates remain, and subsequently locks in the best
arm among the surviving $N$ once its identity becomes clear. To streamline the presentation, we define $R_s$
as the collection of the final $N$ arms that $p_i$ sampled within $\mathcal{A}_i$
via round-robin before until the condition at Line~8 satisfies, $D_s$ as the
subset of arms that $p_i$ discards by virtue of the condition at Line~7, and $j_s$ as
the arm that $p_i$ locks onto from the moment it switches $E_i$ to False until
round $t_{s+1}-1$. With these definitions in hand, we can write

\begin{align*}
&\mathbb{E}\!\left[
\sum_{t=1}^{T}\sum_{a_j}
\mathbf{1}\{\bar{A}_i(t)=a_j\}\cdot \Delta_{i,m_i^*,j}
\mid \neg\mathcal{F}
\right] \\
&\le
\mathbb{E}\!\left[
\sum_{s=0}^{N-1}
\sum_{t=t_s}^{t_{s+1}-1}
\sum_{a_j}
\mathbf{1}\{\bar{A}_i(t)=a_j\}\cdot \Delta_{i,m_i^*,j}
\mid \neg\mathcal{F}
\right] \\
&\le
\mathbb{E}\!\left[
\sum_{s=0}^{N-1}
\sum_{t=t_s}^{t_{s+1}-1}
\left(
\sum_{a_j\in R_s}
\mathbf{1}\{\bar{A}_i(t)=a_j\}\cdot \Delta_{i,m_i^*,j}
\right.\right.\\
&\qquad\left.\left.
+
\sum_{a_j\in D_s}
\mathbf{1}\{\bar{A}_i(t)=a_j\}\cdot \Delta_{i,m_i^*,j}
+
\mathbf{1}\{\bar{A}_i(t)=a_{j_s}\}\cdot \Delta_{i,m_i^*,j_s}
\right)
\mid \neg\mathcal{F}
\right].\\
&\le \mathbb{E}\!\left[
\sum_{s=0}^{N-1}\sum_{t=t_s}^{t_{s+1}-1}
\left(
\sum_{a_j\in R_s}\mathbf{1}\{\bar{A}_i(t)=a_j\}\cdot \Delta_{i,m_i^*,j}
+
\sum_{a_j\in D_s}\mathbf{1}\{\bar{A}_i(t)=a_j\}\cdot \Delta_{i,m_i^*,j}
\right){\mid\neg\mathcal{F}}
\right],
\end{align*}

where the simplification above is justified by Lemma~\ref{lem:ucb_ordering} together with the offline GS structure before termination of the GS procedure,
arm $j_s$ is ranked above $m_i^*$ in player $p_i$'s preference, which means that the regret
from playing these arms is non-positive.

We first handle the contribution from the retained arms in the expression above. By interchanging the order of summation, we get

\begin{align*}
&\mathbb{E}\!\left[
\sum_{s=0}^{N-1}\sum_{t=t_s}^{t_{s+1}-1}\sum_{a_j\in R_s}
\mathbf{1}\{\bar{A}_i(t)=a_j\}\cdot \Delta_{i,m_i^*,j}{\mid\neg\mathcal{F}}
\right] \\
&= \mathbb{E}\!\left[
\sum_{s=0}^{N-1}\sum_{a_j\in R_s}\sum_{t=t_s}^{t_{s+1}-1}
\mathbf{1}\{\bar{A}_i(t)=a_j\}\cdot \Delta_{i,m_i^*,j}{\mid\neg\mathcal{F}}
\right] \\
&\le \mathbb{E}\!\left[
\sum_{s=0}^{N-1}\sum_{a_j\in R_s}
\frac{3\log T}{2}
\left(\frac{4LM}{\Delta_{i,j_s,j_{s+1}}}\right)^{2/\alpha}\cdot \Delta_{i,m_i^*,j}
{\mid\neg\mathcal{F}}
\right]  \\
&\le \frac{3N^2\log T}{2}
\left(\frac{4LM}{\Delta}\right)^{2/\alpha}\cdot \Delta_{i,max}
\end{align*}

where the penultimate step applies Lemma~\ref{lem:sample_complexity}, and the last step uses the observation that $R_s$ has cardinality at most
$N$, a direct consequence of the elimination rule (Line~7).

We now address the contribution from the eliminated arms. Given an arm $a_j$ and
$s\in\{0,\dots,N-1\}$, we write $T_{i,j,s}$ for the value of $T_{i,j}$
upon completion of round $t_{s+1}-1$. When $s\ge1$ and $a_j\in D_s$,
provided that $T_{i,j,s-1}\le \frac{3\log T}{2}
\left(\frac{4LM}{\Delta_{i,j_{s-1},j}}\right)^{2/\alpha}$, we must have
\begin{align*}
T_{i,j,s}:=\sum_{s'\le s}(T_{i,j,s'}-T_{i,j,s'-1})
\le
\frac{3\log T}{2}
\left(\frac{4LM}{\Delta_{i,j_s,j}}\right)^{2/\alpha}
\end{align*}

for arm $a_j$ to be removed from $\mathcal{A}_i$ at step $s$,
which follows from Lemma~\ref{lem:sample_complexity}.

In the complementary case where $T_{i,j,s-1}>\frac{3\log T}{2}
\left(\frac{4LM}{\Delta_{i,j_{s-1},j}}\right)^{2/\alpha}$,
an application of Lemma~\ref{lem:sample_complexity} yields

\begin{align*}
T_{i,j,s}-T_{i,j,s-1}
\le
\frac{3\log T}{2}
\left(\frac{4LM}{\Delta_{i,j_s,j}}\right)^{2/\alpha}
-
\frac{3\log T}{2}
\left(\frac{4LM}{\Delta_{i,j_{s-1},j}}\right)^{2/\alpha}
\end{align*}

at the moment when $a_j$ gets eliminated.

For a given arm $a_j$, we introduce
\begin{align*}
s_{j,1}:=\max_{0\le s\le N-1}
\{s:T_{i,j,s}\le \frac{3\log T}{2}
\left(\frac{4LM}{\Delta_{i,j_s,j}}\right)^{2/\alpha}\}
\end{align*}

which identifies the latest step at which the observation count for $a_j$
has not yet surpassed the required threshold.

Using this definition, the eliminated-arm contribution can be bounded as follows:

\begin{align*}
&\mathbb{E}\!\left[
\sum_{s=0}^{N-1}\sum_{t=t_s}^{t_{s+1}-1}\sum_{a_j\in D_s}
\mathbf{1}\{\bar{A}_i(t)=a_j\}\cdot \Delta_{i,m_i^*,j}
{\mid\neg\mathcal{F}}
\right] \\
&= \mathbb{E}\!\left[
\sum_{s=0}^{N-1}\sum_{a_j\in D_s}
(T_{i,j,s}-T_{i,j,s-1})\cdot \Delta_{i,m_i^*,j}
{\mid\neg\mathcal{F}}
\right] \\
&= \mathbb{E}\!\left[
\sum_{a_j\in\mathcal{K}\setminus\{a_{m_i^*}\}}
\left(
\sum_{s:a_j\in D_s,\,s\le s_{j,1}}
(T_{i,j,s}-T_{i,j,s-1})
\right. \right. \\
&\qquad \left. \left.
+ \sum_{s:a_j\in D_s,\,s>s_{j,1}}
(T_{i,j,s}-T_{i,j,s-1})
\right)
\Delta_{i,m_i^*,j}
{\mid\neg\mathcal{F}}
\right] \\
&\le
\mathbb{E}\!\left[
\sum_{a_j\in\mathcal{K}\setminus\{a_{m_i^*}\}}
\left(
T_{i,j,s_{j,1}}
+
\sum_{s:a_j\in D_s,\,s>s_{j,1}}
(T_{i,j,s}-T_{i,j,s-1})
\right)
\Delta_{i,m_i^*,j}
{\mid\neg\mathcal{F}}
\right] \\
&\le
\sum_{a_j\in\mathcal{K}\setminus\{a_{m_i^*}\}}
\left(
\frac{3\log T}{2}
\left(\frac{4LM}{\Delta_{i,j_{s_{j,1}},j}}\right)^{2/\alpha}
\right. \\
&\qquad \left.
+
\sum_{s:a_j\in D_s,\,s>s_{j,1}}
\left(
\frac{3\log T}{2}
\left(\frac{4LM}{\Delta_{i,j_s,j}}\right)^{2/\alpha}
-
\frac{3\log T}{2}
\left(\frac{4LM}{\Delta_{i,j_{s-1},j}}\right)^{2/\alpha}
\right)
\right)\Delta_{i,m_i^*,j} \\
&\le
\sum_{a_j\in\mathcal{K}\setminus\{a_{m_i^*}\}}
\frac{3\log T}{2}
\left(\frac{4LM}{\Delta_{i,m_i^*,j}}\right)^{2/\alpha}
\cdot \Delta_{i,m_i^*,j} \\
&\le
\sum_{a_j\in\mathcal{K}\setminus\{a_{m_i^*}\}}
\frac{3\log T}{2}
\left(\frac{(4LM)^{2/\alpha}}{\Delta_{i,m_i^*,j}^{{2/\alpha}-1}}\right)
\end{align*}

where the penultimate inequality relies on the definition of $s_{j,1}$ combined with the telescoping argument developed above.

\end{proof}

\begin{lemma}
\label{lem:collision_regret}
The expected regret arising from collisions, conditioned on the good event, satisfies
\[
\mathbb{E}\!\left[
\sum_{t=1}^{T}
\mathbf{1}\{\bar{A}_i(t)=\emptyset\}\cdot \mu_{i,m_i^*}
{\mid\neg\mathcal{F}}
\right]
\le {3N^2\log T}
\left(\frac{4LM}{\Delta}\right)^{2/\alpha}\cdot \Delta_{i,max} .
\]
\end{lemma}

\begin{proof}
Under the Improved \texttt{CPT-ETGS} protocol, whenever $E_i=\text{True}$, the central coordinator distributes arms
from $\mathcal{A}_i$ to player $p_i$ according to a round-robin schedule. Because the pool of arms
$\left|\cup_{i:E_i=\text{True}}\mathcal{A}_i\right|$ available for exploration strictly exceeds the count
$\sum_i \mathbf{1}\{E_i=\text{True}\}$ of active explorers a consequence of the elimination
rule at Line~7 we may safely assume that no collisions arise during the exploration phase, as
elaborated in ~\cite{kong2024improved}. It follows that collision-induced regret is confined to rounds where
$E_i=\text{False}$.

We write $t_s$ and $\bar{t}_s$ for the round indices at which $p_i$ transitions $E_i$ to False for the $s$-th occasion
and back to True for the $(s+1)$-th occasion, respectively. During any period with $E_i=\text{False}$,
$p_i$ consistently proposes to its committed arm $A_i$. Let $j_s$ denote the specific arm that $p_i$ selects
throughout the interval $[t_s, \bar{t}_s]$.

An important feature of the Improved \texttt{CPT-ETGS} algorithm is that each time an arm enters $D_i$ (Line~13),
previously confidence-based eliminated arms may be restored to $\mathcal{A}_i$. This restoration is triggered exclusively when another player
commits to its current top-ranked arm. By Lemma~\ref{lem:gs_proposals}, such restorations occur at most $N$ times in total. For an arbitrary player $p_{i'}$,
let $t'_{i',r}$ mark the round at which $p_{i'}$ augments $D_{i'}$ (Line~13) for the
$r$-th time. The time points $\{t'_{i',r}\}_{r\in[N]}$ further subdivide the intervals $\{[t_s,\bar{t}_s]\}_{s\in[N]}$
into at most $2N$ segments. We denote the boundaries of the $s$-th segment by $t'_s$ and $\bar{t}'_s$,
for $s\in[2N]$. An application of Lemma~\ref{lem:sample_complexity} shows that $p_{i'}$ and $p_i$ propose to
the same arm no more than $\frac{3\log T}{2}
\left(\frac{4LM}{\Delta}\right)^{2/\alpha}.$ times within any individual segment.

Assembling these observations, the collision regret can be controlled as follows:

\begin{align*}
&\mathbb{E}\!\left[
\sum_{t=1}^{T}\mathbf{1}\{\bar{A}_i(t)=\emptyset\}\cdot \mu_{i,m_i^*}
{\mid\neg\mathcal{F}}
\right] \\
&\le
\mathbb{E}\!\left[
\sum_{s=1}^{N}\sum_{t=t_s}^{\bar{t}_s}
\mathbf{1}\{\bar{A}_i(t)=\emptyset\}\cdot \mu_{i,m_i^*}
{\mid\neg\mathcal{F}}
\right] \\
&=
\mathbb{E}\!\left[
\sum_{s=1}^{N}\sum_{t=t_s}^{\bar{t}_s}
\mathbf{1}\{\bar{A}_i(t)=\emptyset, A_i(t)=j_s\}
\cdot \mu_{i,m_i^*}
{\mid\neg\mathcal{F}}
\right] \\
&\le
\mathbb{E}\!\left[
\sum_{i'\ne i}\sum_{s=1}^{N}\sum_{t=t_s}^{\bar{t}_s}
\mathbf{1}\{A_i(t)=A_{i'}(t)=j_s\}
\cdot \mu_{i,m_i^*}
{\mid\neg\mathcal{F}}
\right] \\
&\le
\mathbb{E}\!\left[
\sum_{i'\ne i}\sum_{s=1}^{2N}\sum_{t=t'_s}^{\bar{t}'_s}
\mathbf{1}\{A_i(t)=A_{i'}(t)\}
\cdot \mu_{i,m_i^*}
{\mid\neg\mathcal{F}}
\right] \\
&\le
\sum_{i'\ne i}\sum_{s=1}^{2N}
\frac{3\log T}{2}
\left(\frac{4LM}{\Delta}\right)^{2/\alpha} \cdot \mu_{i,m_i^*} \\
&\le
{3N^2\log T}
\left(\frac{4LM}{\Delta}\right)^{2/\alpha}\cdot \Delta_{i,max}
\end{align*}

\end{proof}

\begin{lemma}[{\cite[Lemma A.3]{kong2024improved}}]
\label{lem:gs_proposals}
Under the offline GS algorithm, the total number of proposals issued by all players
prior to termination is bounded above by $N-1$.
\end{lemma}
\begin{lemma}
\label{lem:failure_prob}
The probability of the concentration failure event is bounded by
\[
\mathbb{P}(\mathcal{F}) \le \frac{2NK}{T}.
\]

\end{lemma}
 \begin{proof}
  The proof is identical to that of the corresponding lemma in Appendix~\ref{app:proof_cpt_etgs},
  using Lemma~\ref{lemma:prob_bound} and a union bound over all $t \in [T]$, $i \in [N]$, $j \in
  [K]$, and $s \in [t]$.
  \end{proof}
\begin{lemma}
\label{lem:sample_complexity}
Consider a player $p_i$ and set $\bar{T}_i =\frac{3\log T}{2}
\left(\frac{4LM}{\Delta}\right)^{2/\alpha}.$
For any pair of arms $j, j'$ satisfying $\mu_{i,j} > \mu_{i,j'}$ with
$\sigma_i(a_j)\in[1,\sigma_i(m_i^*)]$, whenever
\[
T_i(t):=\min\{T_{i,j}(t),T_{i,j'}(t)\}>\bar{T}_i,
\]
it follows that
\[
\text{W-}\mathrm{UCB}_{i,j'}(t) < \text{W-}\mathrm{LCB}_{i,j}(t).
\]
\end{lemma}
\begin{proof}
 Proof is identical to Lemma~\ref{lemma:ti_cpt_etgs}
  \end{proof}

\begin{lemma}
\label{lem:ucb_ordering}
Under the event $\neg\mathcal{F}$, at every round $t$,
\[
\text{W-}\mathrm{UCB}_{i,j}(t) < \text{W-}\mathrm{LCB}_{i,j'}(t)
\quad \text{implies} \quad
\mu_{i,j} < \mu_{i,j'}.
\]
\end{lemma}
 \begin{proof}
  Under $\neg\mathcal{F}$, the true value satisfies $\text{W-}\mathrm{LCB}_{i,j}(t) \le \mu_{i,j} \le
  \text{W-}\mathrm{UCB}_{i,j}(t)$ for all $i,j,t$. Therefore $\mu_{i,j} \le \text{W-}\mathrm{UCB}_{i,j}(t) <
  \text{W-}\mathrm{LCB}_{i,j'}(t) \le \mu_{i,j'}$.
  \end{proof}

\section{Proof of Lemma~\ref{lemma:lower_bound}}
\label{app:lower_bound}
\begin{proof}
We prove the result by reduction to the single-player distorted reward multi-armed bandit problem.
Consider an instance where all arms share the same preference ordering over players, $p_1 \succ p_2 \succ \dots \succ p_N.$
Thus, whenever multiple players propose to the same arm, the arm accepts the player with the highest priority.

Under this configuration, player $p_1$ always has priority over all other players and therefore interacts with the arms without competition. Consequently, the learning problem faced by $p_1$ reduces to a classical stochastic multi-armed bandit problem with $K$ arms and distorted rewards.

For such CPT distorted bandits, Theorem~3 of \cite{kolla2016bandit} establishes that any uniformly efficient algorithm must satisfy
\[
\liminf_{T\to\infty}
\frac{\mathbb{E}[\texttt{CPT-Reg}_1(T)]}{\log T}
\ge
\sum_{j \in [K], j \neq m_1^*}
\frac{(LM)^{2/\alpha}}{4\Delta_{1,m_1^*,j}^{2/\alpha-1}},
\]
which completes the proof.
\end{proof}
\section{Proof of Theorem \ref{th:known}}
\label{app:proof_known}
\begin{lemma}
\label{lemma:bound_C}
Let $\tilde{\mu}_{i,j}$ denote the oracle empirical weight-distorted reward
computed using the stochastic (uncorrupted) rewards $Y^s_{i,j}$, and let
$\hat{\mu}_{i,j}$ denote the empirical weight-distorted reward computed
using the corrupted rewards $Y_{i,j}$.

Assuming rewards are non-negative by Assumption~\ref{ass:positive_rewards}. Then only the positive distortion terms
contribute to the estimator. Under this setting, the difference between
the oracle and corrupted empirical estimators satisfies
\[
\left|\hat{\mu}_{i,j}-\tilde{\mu}_{i,j}\right|
\le
\frac{CL}{T_{i,j}^{\alpha}} .
\]
\end{lemma}

\begin{proof}
The oracle empirical weight-distorted estimator is
\begin{align*}
    \tilde{\mu}_{i,j}
=
\sum_{k=1}^{T_{i,j}}
Y^s_{[i,j,k]}
\left[
w\!\left(\frac{T_{i,j}-k+1}{T_{i,j}}\right)
-
w\!\left(\frac{T_{i,j}-k}{T_{i,j}}\right)
\right],
\end{align*}

while the estimator computed from corrupted rewards is
\begin{align*}
    \hat{\mu}_{i,j}
=
\sum_{k=1}^{T_{i,j}}
Y_{[i,j,k]}
\left[
w\!\left(\frac{T_{i,j}-k+1}{T_{i,j}}\right)
-
w\!\left(\frac{T_{i,j}-k}{T_{i,j}}\right)
\right]
\end{align*}

Taking the absolute difference and applying the triangle inequality gives
  \begin{align*}
  \left|\hat{\mu}_{i,j}-\tilde{\mu}_{i,j}\right|
  &\le
  \sum_{k=1}^{T_{i,j}}
  \left|Y_{[i,j,k]}-Y^s_{[i,j,k]}\right|
  \left|
  w\!\left(\frac{T_{i,j}-k+1}{T_{i,j}}\right)
  -
  w\!\left(\frac{T_{i,j}-k}{T_{i,j}}\right)
  \right|.
  \end{align*}

  Using the H\"older continuity of the distortion function $w(\cdot)$,
  each weight difference satisfies
  \[
  \left|
  w\!\left(\frac{k}{T_{i,j}}\right)
  -
  w\!\left(\frac{k-1}{T_{i,j}}\right)
  \right|
  \le
  \frac{L}{T_{i,j}^{\alpha}} .
  \]
  Since sorting does not increase the total absolute deviation between sequences $\sum_{k=1}^{T_{i,j}} |Y_{[i,j,k]} -
  Y^s_{[i,j,k]}| \le \sum_{k=1}^{T_{i,j}} |Y_{i,j,k} - Y^s_{i,j,k}| \le C$, where the last
  inequality follows from the corruption budget. Combining these bounds:
  \[
  \left|\hat{\mu}_{i,j}-\tilde{\mu}_{i,j}\right|
  \le
  \frac{L}{T_{i,j}^{\alpha}} \cdot C
  =
  \frac{CL}{T_{i,j}^{\alpha}} .
  \]
  \end{proof}

\begin{lemma}
\label{lemma:bound_F_known}
The upper bound for the probability of inaccurately estimating preferences is
\begin{equation*}
\mathbb{P}(\mathcal{F}) \le \frac{2NK}{T}.
\end{equation*}
\end{lemma}
\begin{proof}
\begin{align*}
\mathbb{P}(\mathcal{F})
&\le
\mathbb{P}\!\left(
\exists t,i,j:
\left|\hat{\mu}_{i,j}(t)-\mu_{i,j}\right|
>
LM\!\left(
\frac{3\log T}{2T_{i,j}(t)}
\right)^{\alpha/2}+\frac{CL}{T_{i,j}^{\alpha}(t)}
\right)
\nonumber\\
&\le
\mathbb{P}\!\left(
\exists t,i,j:
\left|\tilde{\mu}_{i,j}(t)-\mu_{i,j}\right|
+
\left|\hat{\mu}_{i,j}(t)-\tilde{\mu}_{i,j}\right|
>LM\!\left(
\frac{3\log T}{2T_{i,j}(t)}
\right)^{\alpha/2}
+\frac{CL}{T_{i,j}^{\alpha}(t)}
\right)
\nonumber\\
&\le
\mathbb{P}\!\left(
\exists t,i,j:
\left|\tilde{\mu}_{i,j}(t)-\mu_{i,j}\right|
>
LM\!\left(
\frac{3\log T}{2T_{i,j}(t)}
\right)^{\alpha/2}
\right) \\
&+
\mathbb{P}\!\left(
\exists t,i,j:
\left|\hat{\mu}_{i,j}(t)-\tilde{\mu}_{i,j}\right|
>
\frac{CL}{T_{i,j}^{\alpha}(t)}
\right)
\nonumber\\
&=
\mathbb{P}\!\left(
\exists t,i,j:
\left|\tilde{\mu}_{i,j}(t)-\mu_{i,j}\right|
>
LM\!\left(
\frac{3\log T}{2T_{i,j}(t)}
\right)^{\alpha/2}
\right)
\nonumber\\
&\le
\sum_{t}\sum_{i}\sum_{j}\sum_{s=1}^{t}
\mathbb{P}\!\left(
T_{i,j}(t)=s,
\left|\tilde{\mu}_{i,j}(t)-\mu_{i,j}\right|
>
LM\!\left(
\frac{3\log T}{2T_{i,j}(t)}
\right)^{\alpha/2}
\right)
\nonumber\\
&\le
\sum_{t\in[T]}\sum_{i\in[N]}\sum_{j\in[K]}
t \cdot 2\exp(-3\ln T)
\nonumber\\
&\le
\frac{2NK}{T},
\end{align*}
The equality follows from the definition of the corruption budget $C$, which ensures that 
$\left|\hat{\mu}_{i,j}(t)-\tilde{\mu}_{i,j}\right|\le \frac{CL}{T_{i,j}^{\alpha}(t)}$ by Lemma~\ref{lemma:bound_C}. 
The exponential bound is obtained by using Lemma~\ref{lemma:prob_bound}

\end{proof}

\begin{lemma}
Conditional on $\neg \mathcal{F}$,
\[
\mathrm{\text{W-UCB}}_{i,j}(t) < \mathrm{\text{W-LCB}}_{i,j'}(t)
\quad \Rightarrow \quad
\mu_{i,j} < \mu_{i,j'} .
\]
\end{lemma}

\begin{proof}
Condition on the event $\neg\mathcal{F}$. Then for every $t \in [T]$, $i \in [N]$, and $j \in [K]$, the estimator $\tilde{\mu}_{i,j}(t)$ satisfies
\[
\bigl|\tilde{\mu}_{i,j}(t) - \mu_{i,j}\bigr|
\le
LM\!\left(
\frac{3\log T}{2T_{i,j}(t)}
\right)^{\alpha/2}.
\]

We first derive a lower bound on the confidence interval. Using the definition of $\mathrm{LCB}_{i,j}(t)$ and Lemma~\ref{lemma:bound_C}, we obtain
\begin{align*}
\mathrm{LCB}_{i,j}(t)
&=
\hat{\mu}_{i,j}(t)
-
LM\!\left(
\frac{3\log T}{2T_{i,j}(t)}
\right)^{\alpha/2}
-
\frac{CL}{T_{i,j}^{\alpha}(t)} \\
&\le
\tilde{\mu}_{i,j}(t)
-
LM\!\left(
\frac{3\log T}{2T_{i,j}(t)}
\right)^{\alpha/2}
\le
\mu_{i,j}.
\end{align*}

Similarly, for the upper confidence bound we have
\begin{align*}
\mu_{i,j}
&\le
\tilde{\mu}_{i,j}(t)
+
LM\!\left(
\frac{3\log T}{2T_{i,j}(t)}
\right)^{\alpha/2} \\
&\le
\hat{\mu}_{i,j}(t)
+
LM\!\left(
\frac{3\log T}{2T_{i,j}(t)}
\right)^{\alpha/2}
+
\frac{CL}{T_{i,j}^{\alpha}(t)}
=
\mathrm{UCB}_{i,j}(t).
\end{align*}

Therefore, under $\neg\mathcal{F}$, the true value $\mu_{i,j}$ lies inside the confidence interval $[\mathrm{LCB}_{i,j}(t),\mathrm{UCB}_{i,j}(t)]$. Consequently,
\[
\mu_{i,j} \le \mathrm{UCB}_{i,j}(t)
<
\mathrm{LCB}_{i,j'}(t)
\le
\mu_{i,j'} .
\]

This completes the proof.
\end{proof}

\begin{lemma}
\label{lemma:ti_known}
In round $t$, let $T_i(t)=\min_{j\in[K]}T_{i,j}(t)$ and
\[
\bar{T}_i=\frac{3\log T}{2\,\Delta^{2/\alpha}}(8LM)^{\frac{2}{\alpha}}  + \left(\frac{8CL}{\Delta}\right)^{\frac{1}{\alpha}}
\]
Conditional on $\neg\mathcal{F}$, if $T_i(t)>\bar{T}_i$, we have
\[
\mathrm{\text{W-LCB}}_{i,\rho_{i,k}}(t)>\mathrm{\text{W-UCB}}_{i,\rho_{i,k+1}}(t), \quad \forall k\in[N],
\]
and
\[
\mathrm{\text{W-LCB}}_{i,\rho_{i,N}}(t)>\mathrm{\text{W-UCB}}_{i,\rho_{i,k}}(t), \quad \forall k=N+1,N+2,\dots,K.
\]
\end{lemma}

\begin{proof}
Assume for contradiction that the ordering induced by the confidence bounds is violated. 
That is, there exists $k \in [N]$ such that
\[
\mathrm{W\text{-}LCB}_{i,\rho_{i,k}}(t)
\le
\mathrm{W\text{-}UCB}_{i,\rho_{i,k+1}}(t),
\]
or there exists $k \in \{N+1,\dots,K\}$ for which
\[
\mathrm{W\text{-}LCB}_{i,\rho_{i,N}}(t)
\le
\mathrm{W\text{-}UCB}_{i,\rho_{i,k}}(t).
\]

Let $j'$ denote the arm corresponding to the left-hand side of the inequality and $j$ the arm on the right-hand side.

Condition on the event $\neg \mathcal{F}$. By the definition of the weighted confidence bounds, the true means satisfy
\begin{align*}
\mathrm{W\text{-}LCB}_{i,j'}(t)
&\ge
\mu_{i,j'}
-
2LM\!\left(
\frac{3\log T}{2T_{i,j}(t)}
\right)^{\alpha/2}
-
\frac{2CL}{T_i^{\alpha}(t)}, \\
\mathrm{W\text{-}UCB}_{i,j}(t)
&\le
\mu_{i,j}
+
2LM\!\left(
\frac{3\log T}{2T_{i,j}(t)}
\right)^{\alpha/2}
+
\frac{2CL}{T_i^{\alpha}(t)} .
\end{align*}

Combining these inequalities with the assumed violation of the ordering yields
\[
\mu_{i,j'}-\mu_{i,j}
\le
4LM\!\left(
\frac{3\log T}{2T_{i,j}(t)}
\right)^{\alpha/2}
+
\frac{4CL}{T_i^{\alpha}(t)} .
\]

Recall that $\Delta_{i,j,j'} = \mu_{i,j'}-\mu_{i,j}$. Hence
\[
\Delta_{i,j,j'}
\le
4LM\!\left(
\frac{3\log T}{2T_{i,j}(t)}
\right)^{\alpha/2}
+
\frac{4CL}{T_i^{\alpha}(t)} .
\]

Now consider the case where
\[
T_i(t) >
\frac{3\log T}{2\,\Delta^{2/\alpha}}(8LM)^{2/\alpha}
+
\left(\frac{8CL}{\Delta}\right)^{1/\alpha}.
\]

Substituting this lower bound into the right-hand side above gives
\[
\Delta_{i,j,j'}
<
\frac{\Delta}{2} + \frac{\Delta}{2}
=
\Delta .
\]

However, by definition the gap satisfies $\Delta_{i,j,j'} \ge \Delta$, which leads to a contradiction. Therefore the assumed violation cannot occur, and the ordering induced by the confidence bounds must hold whenever $T_i(t)>\bar{T}_i$.
\end{proof}

\begin{proof}
\textit{(of Theorem~\ref{th:known})}   
In Algorithm~\ref{alg:cpt_etgs}, players propose to arms in a round-robin manner during Phase~2. 
As a result, no collisions occur, and every player is successfully matched in each exploration round. 
Consequently, all players transition to Phase~3 simultaneously. 
Let $\ell_{\max}$ denote the largest sub-phase index of Phase~2, after which the players' preference rankings are estimated accurately.
The optimal stable regret for player i can be bounded as follows,

\begin{align*}
\text{CPT-}\text{Reg}_i(T)
&= \mathbb{E}\!\left[\sum_t (\mu_{i,m_i^*}-\mu_{i,\bar{A}_i(t)})\right] \\
&\le \mathbb{E}\!\left[\sum_t \mathbb{I}\{\bar{A}(t)\neq m^*\}\cdot \Delta_{i,\max}\right] \\
&\le N\Delta_{i,\max}
+ \mathbb{E}\!\left[\sum_{t=N+1}^{T}\mathbb{I}\{\bar{A}(t)\neq m^*\}\,\middle|\, \neg\mathcal{F}\right]\cdot \Delta_{i,\max}
+ \mathbb{P}(\mathcal{F})\cdot T \cdot \Delta_{i,\max} \\
&\le N\Delta_{i,\max}
+ \mathbb{E}\!\left[\sum_{k=1}^{L_{\max}} (2^k+1) + N^2\,\middle|\, \neg\mathcal{F}\right]\cdot \Delta_{i,\max}
+ \mathbb{P}(\mathcal{F})\cdot T\cdot \Delta_{i,\max} \\
&\le
N\Delta_{i,\max}
+ \mathbb{E}\!\left[\sum_{k=1}^{L_{\max}} (2^k+1) + N^2\middle|\, \neg\mathcal{F}\right]\cdot \Delta_{i,\max}
+ 2NK\cdot \Delta_{i,\max}.
\end{align*}
The last inequality follows from Lemma~\ref{lemma:bound_F_known}.

From Lemma~\ref{lemma:ti_known}, Phase~2 terminates after at most $L_{\max}$ sub-phases, where
\[
L_{\max}
=
\min\left\{
k :
\sum_{k'=1}^{k} d_{k'}
\ge
\frac{3K\log T}{2\,\Delta^{2/\alpha}}(8LM)^{\frac{2}{\alpha}}
+
K\left(\frac{8CL}{\Delta}\right)^{\frac{1}{\alpha}}
\right\}.
\]

Recall that in Algorithm~\ref{alg:cpt_etgs} the exploration length of the $k$-th sub-phase grows geometrically, i.e., $d_k = 2^k$. Using this definition, the cumulative exploration length up to sub-phase $L_{\max}$ satisfies
\[
\sum_{k'=1}^{L_{\max}} 2^{k'}
\le
\frac{3K\log T}{\Delta^{2/\alpha}}(8LM)^{\frac{2}{\alpha}}
+
2K\left(\frac{8CL}{\Delta}\right)^{\frac{1}{\alpha}} .
\]

Consequently, the number of sub-phases can be bounded by
\[
L_{\max}
=
\log\!\left(
\frac{3K\log T}{\Delta^{2/\alpha}}(8LM)^{\frac{2}{\alpha}}
+
2K\left(\frac{8CL}{\Delta}\right)^{\frac{1}{\alpha}}
\right).
\]

Using this bound on $L_{\max}$, the regret of player $i$ can be bounded as
\begin{align*}
\text{CPT-}\text{Reg}_i(T)
&\le
\Bigg(
\frac{(8LM)^{2/\alpha}\, 3K\log T}{\Delta^{2/\alpha}}
+
2K\left(\frac{8CL}{\Delta}\right)^{1/\alpha}
\\
&\qquad
+
\log\!\Bigg(
\frac{(8LM)^{2/\alpha}\, 3K\log T}{\Delta^{2/\alpha}}
+
2K\left(\frac{8CL}{\Delta}\right)^{1/\alpha}
\Bigg)
\Bigg)\Delta_{i,\max}
\\
&\qquad
+
N\Delta_{i,\max}
+
N^2\Delta_{i,\max}
+
2NK\Delta_{i,\max}.
\end{align*}
\end{proof}

\section{Proof of Theorem \ref{th:unknown}}
 \label{app:proof_unknown}
For every layer $\ell \in [\log T]$, we use the notation $\hat{\mu}^{\ell}_{i,j}(t)$,
$T^{\ell}_{i,j}(t)$, $\mathrm{UCB}^{\ell}_{i,j}(t)$, $\mathrm{LCB}^{\ell}_{i,j}(t)$
to represent the values of $\hat{\mu}^{\ell}_{i,j}$, $T^{\ell}_{i,j}$,
$\mathrm{UCB}^{\ell}_{i,j}$ and $\mathrm{LCB}^{\ell}_{i,j}$ upon completion of
round $t$, respectively. We introduce the following event:
\[
\mathcal{F}_S^{\ell} =
\left\{
\exists t\in[T], i\in[N], j\in[K] :
\left|\hat{\mu}^{\ell}_{i,j}(t)-\mu_{i,j}\right|
>
LM\!\left(
\frac{3\log T}{2T_{i,j}(t)}\right)^{\alpha/2} +\frac{2dL\log T}{(T^{\ell}_{i,j}(t))^\alpha}
\right\}
\]
which captures the scenario where some preference estimate deviates significantly from its true value over the time horizon
for layer $\ell$. Additionally, let $\mathcal{F}_C$ denote the event
that there exists an ETGS instance across layers $\ell$,
$\ell' \in [\log T]$ with sampling probability satisfying
$2^{-\ell'} \le 1/C$, such that the cumulative corruption actually incurred by some
player within this instance exceeds $d\log T + 2$.
We then define the overall failure event as
\[
\mathcal{F} =
\left(\bigcup_{\ell \in [\log T]} \mathcal{F}_S^{\ell}\right)
\cup \mathcal{F}_C .
\]

We now proceed to derive an upper bound on the probability of the
failure event $\mathcal{F}$.

\medskip
\begin{lemma}[{\cite[Lemma3.3]{wubandit}}]
\label{lemma:robust}
In Algorithm~\ref{alg:etgs_unknown}, the ETGS instance with sampling probability less than $1/C$ experiences,
w.p.\ at least $1 - 1/T$, cumulative corruption bounded by $d \log(T) + 2$ during the exploration phase.
\end{lemma}
\noindent
\begin{lemma}
\label{lemma:bound_unknown}
The probability of the failure event $\mathcal{F}$ satisfies the following upper bound:
\begin{equation*}
\mathbb{P}(\mathcal{F})
\le
\frac{2NK\log T}{T}
+
\frac{N\log T}{T}.
\end{equation*}

\noindent
\end{lemma}
\begin{proof}
We begin by introducing the notation $C_i^{\ell}(T)$ to denote the total corruption
experienced by player $i$ in the $\ell$-th ETGS instance. Since the
constant $d$ is chosen so that $d\log T + 2 \le 2d\log T$, the probability of the failure event $\mathcal{F}$ can be bounded as follows:
\begin{align*}
\mathbb{P}(\mathcal{F})
&\le
\sum_{\ell\in[\log T]}
\mathbb{P}(\mathcal{F}_S^{\ell})
+
\mathbb{P}(\mathcal{F}_C)
\\
&\le
\sum_{\ell\in[\log T]}
\mathbb{P}\!\left(
\exists t,i,j :
\left|
\hat{\mu}^{\ell}_{i,j}(t)-\mu_{i,j}
\right|
>
LM\!\left(
\frac{3\log T}{2T_{i,j}(t)}\right)^{\alpha/2} +\frac{2dL\log T}{(T^{\ell}_{i,j}(t))^\alpha}
\right)
\nonumber\\
&\quad+
\sum_{\ell:2^{-\ell}\le1/C}
\mathbb{P}\!\left(
\exists i : C_i^{\ell}(T) > 2d\log(T)
\right)
\\
&\le
\sum_{\ell\in[\log T]}
\sum_{t\in[T]}
\sum_{i\in[N]}
\sum_{j\in[K]}
t\cdot 2\exp(-3\ln T)
+
\sum_{i\in[N]}
\frac{\log T}{T}
\nonumber
\\
&\le
\frac{2NK\log T}{T}
+
\frac{N\log T}{T}.
\nonumber
\end{align*}

\medskip

By virtue of Lemma~\ref{lemma:robust}, we can conclude that the actual corruption
sustained by layers $\ell$ for which $2^{\ell} \ge C$,
$\ell\in[\log T]$, does not exceed $d\log T + 2$. Given that the constant
$d$ is selected to ensure $d\log T + 2 \le 2d\log T$, and the confidence interval
is set to
$
LM\!\left(
\frac{3\log T}{2T_{i,j}(t)}\right)^{\alpha/2} +\frac{2dL\log T}{(T^{\ell}_{i,j}(t))^\alpha}
$
the following results are an immediate consequence.
\end{proof}
\begin{lemma}
\label{lemma:ti_unknown}
At round $t$, define
\[
T_i^\ell(t)=\min_{j\in[K]} T_{i,j}^\ell(t)
\]
together with
\[
\bar T_i^\ell=\frac{3\log T}{2}\left(\frac{8LM}{\Delta}\right)^{\frac{2}{\alpha}} + \left(\frac{16dL\log T}{\Delta}\right)^{\frac{1}{\alpha}}
\]
Under the event $\neg \mathcal{F}$, whenever $T_i^\ell(t)>\bar T_i^\ell$, the following separation holds:
\[
\mathrm{LCB}_{i,\rho_{i,j}}^\ell(t) >
\mathrm{UCB}_{i,\rho_{i,j+1}}^\ell(t)
\quad \text{for any } j\in[N],
\]
and moreover,
\[
\mathrm{LCB}_{i,\rho_{i,N}}^\ell(t) >
\mathrm{UCB}_{i,\rho_{i,j}}^\ell(t)
\quad \text{for any } j=N+1,N+2,\dots,K .
\]

\end{lemma}
\begin{proof}
  The argument follows the same structure as the proof of Lemma~\ref{lemma:ti_known}. Assume for
  contradiction that the ordering is violated. Under $\neg\mathcal{F}$, expanding the confidence
  bounds gives
  \[
  \Delta_{i,j,j'} \le 4LM\!\left(\frac{3\log T}{2T_i^\ell(t)}\right)^{\alpha/2} + \frac{8dL\log
  T}{(T_i^\ell(t))^\alpha}.
  \]
  Substituting $T_i^\ell(t) > \bar{T}_i^\ell = \frac{3\log
  T}{2}\left(\frac{8LM}{\Delta}\right)^{2/\alpha} + \left(\frac{16dL\log
  T}{\Delta}\right)^{1/\alpha}$ into each term yields $\Delta_{i,j,j'} < \Delta/2 + \Delta/2 =
  \Delta$, contradicting $\Delta_{i,j,j'} \ge \Delta$.
  \end{proof}

\begin{proof}
\textit{(of Theorem~\ref{th:unknown})} 
The stable regret for player $i$ under the optimal matching admits the following decomposition:
\begin{align*}
\text{CPT-}\mathrm{Reg}_i(T)
&= \mathbb{E}\!\left[\sum_t (\mu_{i,m_i^*}-\mu_{i,\bar{A}_i(t)})\right] \\
&\le \mathbb{E}\!\left[\sum_t \mathbf{1}\{\bar A(t)\neq m^*\}\cdot\Delta_{i,\max}\right] \\
&\le N\Delta_{i,\max}
+ \mathbb{E}\!\left[\sum_{t=N+1}^{T}\mathbf{1}\{\bar A(t)\neq m^*\}\mid\mathcal{F}^c\right]\Delta_{i,\max}
+ \mathbb{P}(\mathcal{F})\,T\,\Delta_{i,\max} \\
&\le (N+2NK\log T + N\log T)\Delta_{i,\max} \\
&+ \mathbb{E}\!\left[\sum_{t=N+1}^{T}\mathbf{1}\{\bar A(t)\neq m^*\}\mid\mathcal{F}^c\right]\Delta_{i,\max},
\end{align*}

where the final inequality is obtained by applying Lemma~\ref{lemma:bound_unknown}.

We now analyze separately the regret arising from layers whose actual corruption remains below $C$, and the regret from layers that lack sufficient robustness to handle corruption level $C$. To this end, we introduce the smallest layer index that provides robustness against corruption:
\[
\ell^* := \arg\min_{\ell}\{2^\ell > C\}.
\]

For the robust layers $\ell' \ge \ell^*$, combining Lemma~\ref{lemma:robust} with Lemma~\ref{lemma:bound_unknown} yields a regret upper bound of
\[
\left(\frac{3K\log T}{2}\left(\frac{8LM}{\Delta}\right)^{\frac{2}{\alpha}} + K\left(\frac{16dL\log T}{\Delta}\right)^{\frac{1}{\alpha}}\right)\Delta_{i,\max}.
\]

As the total number of layers is bounded by $\log T$, the aggregate regret from all robust layers satisfies
\[
K\log T\left(\frac{3\log T}{2}\left(\frac{8LM}{\Delta}\right)^{\frac{2}{\alpha}} + \left(\frac{16dL\log T}{\Delta}\right)^{\frac{1}{\alpha}}\right)\Delta_{i,\max}.
\]

Turning to the layers $\ell<\ell^*$ that cannot tolerate corruption level $C$ (i.e., those with
$2^\ell < C$), we observe that the stable matchings computed by these faster layers are eventually overridden to coincide with that of layer $\ell^*$, once the latter has accurately estimated its preference rankings.

In expectation, this synchronization cost is at most $C\cdot T_{i,j}^{\ell^*}$, since each action taken by layer $\ell^*$ occurs with probability $1/C$ among the moves that constitute arm $j$ matchings under the optimal stable matching for the faster algorithms.

More precisely, let $T_{i,j}^{\ell^*}$ denote the number of rounds in which the optimal layer is active. In expectation, the layers $\ell<\ell^*$ complete their preference rank estimation before layer $\ell^*$ does. Consequently, we obtain

\begin{equation*}
\mathbb{E}[\tilde T_{i,j}]
= \frac{1}{P_{\ell^*}}T_{i,j}^{\ell^*} - T_{i,j}^{\ell^*}
= C T_{i,j}^{\ell^*} - T_{i,j}^{\ell^*}
\le C T_{i,j}^{\ell^*}.
\end{equation*}

where $\tilde T_{i,j}$ represents the total number of rounds in which layers $\ell<\ell^*$ are selected.

Combining the above, the total regret attributable to the faster layers is bounded by

\[
KC\left(\frac{3\log T}{2}\left(\frac{8LM}{\Delta}\right)^{\frac{2}{\alpha}} + \left(\frac{16dL\log T}{\Delta}\right)^{\frac{1}{\alpha}}\right)\Delta_{i,\max}.
\]

Following the structure of our algorithm, the number of sub-phases within Phase~2 is at most
\[
\left(
\begin{aligned}
&\frac{K}{d}\log T
\left(
\frac{3\log T}{2}\left(\frac{8LM}{\Delta}\right)^{\frac{2}{\alpha}}
+
\left(\frac{16dL\log T}{\Delta}\right)^{\frac{1}{\alpha}}
\right)\Delta_{i,\max} \\
&+
\frac{KC}{d}
\left(
\frac{3\log T}{2}\left(\frac{8LM}{\Delta}\right)^{\frac{2}{\alpha}}
+
\left(\frac{16dL\log T}{\Delta}\right)^{\frac{1}{\alpha}}
\right)\Delta_{i,\max}
\end{aligned}
\right)
\]

which leads to an overall communication cost of

\[
\left(
\begin{aligned}
&\frac{K\log^2 T}{d}
\left(
\frac{3\log T}{2}\left(\frac{8LM}{\Delta}\right)^{\frac{2}{\alpha}}
+
\left(\frac{16dL\log T}{\Delta}\right)^{\frac{1}{\alpha}}
\right)\Delta_{i,\max} \\
&+
\frac{KC\log T}{d}
\left(
\frac{3\log T}{2}\left(\frac{8LM}{\Delta}\right)^{\frac{2}{\alpha}}
+
\left(\frac{16dL\log T}{\Delta}\right)^{\frac{1}{\alpha}}
\right)\Delta_{i,\max}
\end{aligned}
\right)
\]

Collecting all terms, the regret upper bound for the Multi-layer ETGS algorithm is given by

\begin{align*}
\text{CPT-}\mathrm{Reg}_i(T)
&\le (N+N^2\log T + 2NK\log T + N\log T)\Delta_{i,\max} \\
&\quad +K\log T\left(\frac{3\log T}{2}\left(\frac{8LM}{\Delta}\right)^{\frac{2}{\alpha}} + \left(\frac{16dL\log T}{\Delta}\right)^{\frac{1}{\alpha}}\right)\Delta_{i,\max} \\
&\quad +KC\left(\frac{3\log T}{2}\left(\frac{8LM}{\Delta}\right)^{\frac{2}{\alpha}} + \left(\frac{16dL\log T}{\Delta}\right)^{\frac{1}{\alpha}}\right)\Delta_{i,\max} \\
&\quad + \frac{K\log^2 T}{d}
\left(
\frac{3\log T}{2}\left(\frac{8LM}{\Delta}\right)^{\frac{2}{\alpha}}
+
\left(\frac{16dL\log T}{\Delta}\right)^{\frac{1}{\alpha}}
\right)\Delta_{i,\max} \\
&+
\frac{KC\log T}{d}
\left(
\frac{3\log T}{2}\left(\frac{8LM}{\Delta}\right)^{\frac{2}{\alpha}}
+
\left(\frac{16dL\log T}{\Delta}\right)^{\frac{1}{\alpha}}
\right)\Delta_{i,\max} \\
&=
O\!\Big(
\big[
KC\!\left(\log T+\frac{\log^2T}{d}\right)
\left(\frac{LM}{\Delta}\right)^{\frac{2}{\alpha}}  \\
&\qquad
+
K(\log T+C)\!\left(1+\frac{\log T}{d}\right)
(\log T)^{\frac{1}{\alpha}}
\left(\frac{dL}{\Delta}\right)^{\frac{1}{\alpha}}
\big]
\Delta_{i,\max}
\Big).
\end{align*}
\end{proof}
\end{document}